\documentclass[11pt]{article}
\usepackage[a4paper,top=3cm,bottom=2cm,left=3cm,right=3cm,marginparwidth=1.75cm]{geometry}
\usepackage{xcolor}
\definecolor{mydarkblue}{rgb}{0.03,0.2,0.4}
\usepackage{xcolor}
\usepackage[colorlinks = true,
            linkcolor = mydarkblue,
            urlcolor  = mydarkblue,
            citecolor = mydarkblue,
            anchorcolor = mydarkblue]{hyperref}
\usepackage[round]{natbib}

\usepackage[normalem]{ulem}
\usepackage[none]{hyphenat}
\usepackage{breakcites}
\usepackage[utf8]{inputenc} 
\usepackage[T1]{fontenc}    
\usepackage{hyperref}       
\usepackage{url}            
\usepackage{booktabs}       
\usepackage{amsfonts}       
\usepackage{nicefrac}       
\usepackage{microtype}      
\usepackage{lipsum}
\usepackage[linesnumbered,ruled,vlined]{algorithm2e}

\SetCommentSty{mycommfont}
\SetKwInput{KwInput}{Input}                
\SetKwInput{KwOutput}{Output}  
\SetKwInput{KwInitialization}{Initialization}
\usepackage{float}
\usepackage{caption}
\usepackage{subcaption}
\usepackage{mathtools}
\usepackage{soul}
\mathtoolsset{showonlyrefs}
\usepackage{enumitem}
\usepackage{amssymb}
\usepackage{amsthm}
\usepackage{tikz}
\usepackage{bm}
\usetikzlibrary{arrows}
\usepackage{dsfont}
\usepackage{graphicx, graphics}
\usepackage{wrapfig}
\usepackage{thmtools,thm-restate}

\newtheorem{theorem}{Theorem}

\newtheorem{assumption}{Assumption}
\newtheorem{definition}{Definition}

\newtheorem{proposition}{Proposition}

\DeclareMathAlphabet{\mathbsf}{OT1}{cmss}{bx}{n}
\DeclareMathAlphabet{\mathssf}{OT1}{cmss}{m}{sl}
\title{Weakly Supervised Representation Learning with Sparse Perturbations}

\author{%
Kartik Ahuja, Jason Hartford \& Yoshua Bengio\footnote{CIFAR Senior Fellow and CIFAR AI Chair}  \\
Mila - Quebec AI Institute, Universit\'e de Montr\'eal\\
Quebec, Canada \\
\texttt{\{kartik.ahuja,jason.hartford,yoshua.bengio\}@mila.quebec}
}
\date{}
\begin{document}
\maketitle

\begin{abstract}

The theory of representation learning aims to build methods that provably invert the data generating process with minimal domain knowledge or any source of supervision. Most prior approaches require strong distributional assumptions on the latent variables and weak supervision (auxiliary information such as timestamps) to provide provable identification guarantees. In this work, we show that if one has weak supervision from observations generated by sparse perturbations of the latent variables--e.g. images in a reinforcement learning environment where actions move individual sprites--identification is achievable under unknown continuous latent distributions. We show that if the perturbations are applied only on mutually exclusive blocks of latents, we identify the latents up to those blocks. We also show that if these perturbation blocks overlap, we identify latents up to the smallest blocks shared across perturbations. Consequently, if there are blocks that intersect in one latent variable only, then such latents are identified up to permutation and scaling. We propose a natural estimation procedure based on this theory and illustrate it on low-dimensional synthetic and image-based experiments.

\end{abstract}

\section{Introduction}

If you are reading this paper on a computer, press one of the arrow keys... all the text you are reading jumps as the screen refreshes in response to your action. Now imagine you were playing a video game like Atari's Space Invaders---the same keystroke would cause a small sprite at the bottom of your screen to move in response. These actions induce changes in pixels that are very different, but in both cases, the visual feedback in response to our actions indicates the presence of some object on the screen---a virtual paper and a virtual spacecraft, respectively---with properties that we can manipulate. Our keystrokes induce sparse changes to a program's state, and these changes are reflected on the screen, albeit not necessarily in a correspondingly sparse way (e.g., most pixels change when scrolling). Similarly, many of our interactions with the real world induce sparse changes to the underlying causal factors of our environment: lift a coffee cup and the cup moves, but not the rest of the objects on your desk; turn your head laterally, and the coordinates of all the objects in the room shift, but only in the horizontal direction. These examples hint at the main question we aim to answer in this paper: if we know that actions have sparse effects on the latent factors of our system, can we use that knowledge as weak supervision to help disentangle these latent factors from pixel-level data?

Self-- and weakly-supervised learning approaches have made phenomenal progress in the last few years, with large-scale systems like GPT-3 \citep{brown2020language} offering large improvements on all natural language benchmarks, and CLIP \citep{radford2021} outperforming state-of-the-art supervised models from six years ago \citep{Szegedy16} on the ImageNet challenge \citep{deng2009imagenet} without using any of the labels.

Yet, despite these advances, these systems are still far from human reasoning abilities and often fail on out-of-distribution examples \citep{Geirhos2020}. To robustly generalize out of distribution, 
we need models that can infer the causal mechanisms that relate latent variables \citep{scholkopf2021toward, scholkopf2022statistical} because these mechanisms are invariant under distribution shift.
The field of causal inference has developed theory and methods to infer causal mechanisms from data \citep{pearl2009causality, peters2017elements}, but these methods assume access to high-level abstract features, instead of low-level signal data such as video, text and images. We need representation learning methods that reliably recover these abstract features if we are to bridge the gap between causal inference and deep learning.

\begin{figure}
	\centering
	\includegraphics[width=\textwidth]{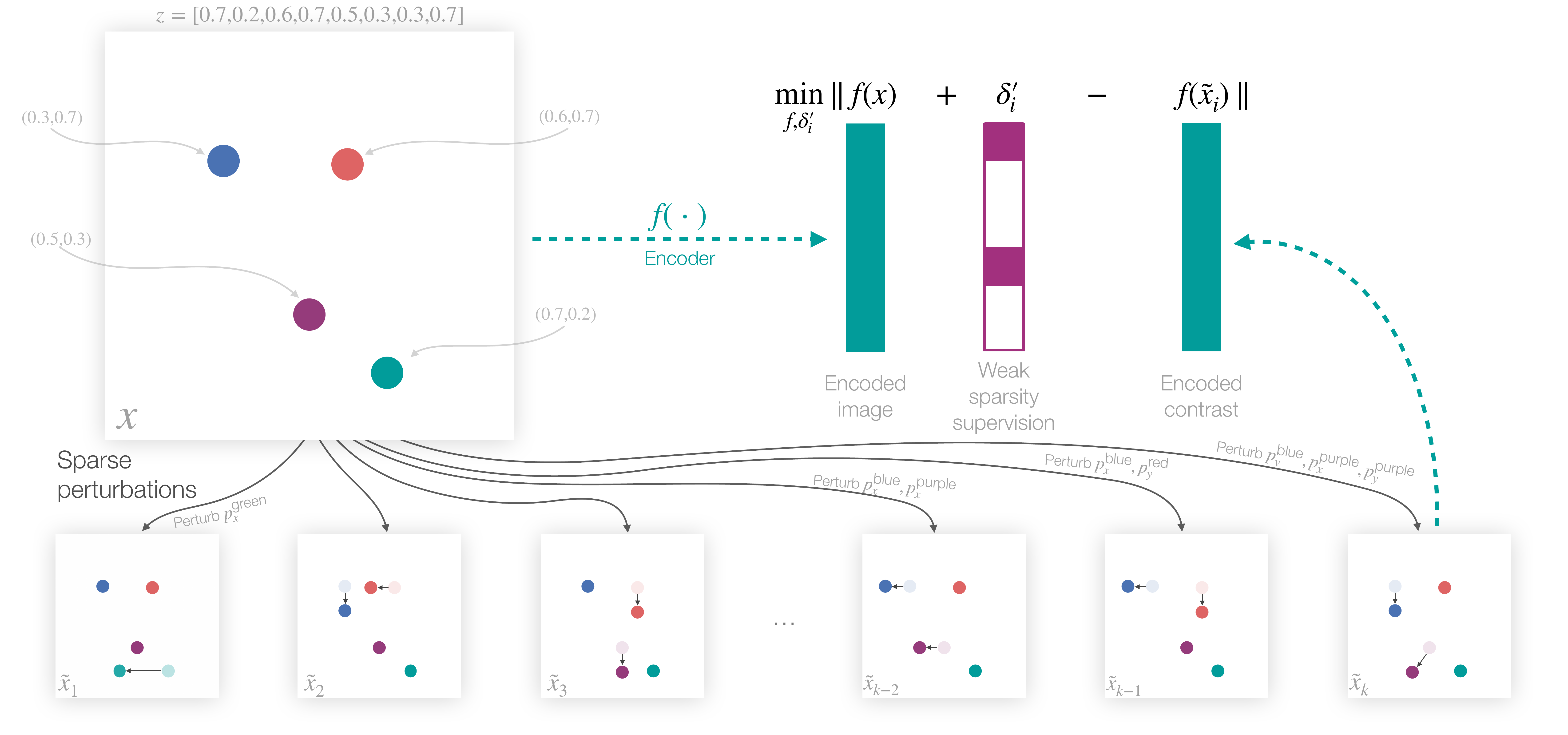}
	\caption{Ball agent interaction environment. Different frames show the effect of sparse perturbations.}
	\label{fig: ball_env}
\end{figure}

This is a challenging task because the problem of inferring latent variables is not identified with independent and identically distributed (IID) data \citep{hyvarinen1999nonlinear, locatello2019challenging}, even in the limit of an infinite number of such IID examples. 
However, there has been significant recent progress in developing representation learning approaches that provably recover latent factors $Z$ (e.g., object positions, object colors, etc.) underlying complex data $X$ (e.g. image), where $X \leftarrow g(Z)$, by going beyond the IID setting and using observations of $X$ along with minimal domain knowledge and supervision  \citep{hyvarinen2016unsupervised, hyvarinen2017nonlinear,locatello2020weakly, khemakhem2020variational}. These works establish provable identification of latents by leveraging strong structural assumptions such as independence conditional on auxiliary information (e.g., timestamps). In this work, we aim to relax these distributional assumptions on the latent variables to achieve identification for arbitrary continuous latent distributions. Instead of distributional assumptions, we assume access to data generated under sparse perturbations that change only a few latent variables at a time as a source of weak supervision. Figure \ref{fig: ball_env} illustrates our working example of this assumption: a simple environment where an agent's actions perturb the coordinates of a few balls at a time. Our main contributions are summarized as follows.

\begin{itemize}
	\item We show that sparse perturbations that impact one latent at a time are sufficient to learn the latents (up to permutation and scaling) that follow any unknown continuous distribution.   
	\item Next, we consider more general settings, where perturbations  impact one block of latent variables at a time. In the setting where blocks do not overlap, we recover the latents up to an affine transformation of these blocks. 
	\item Further, we show that when perturbation blocks overlap, we get stronger identification. In this setting, we prove identification up to affine transformation of the smallest intersecting block. Consequently, if there are blocks that intersect in one latent variable only, then such latents are identified up to permutation and scaling. 
	\item We leverage these results to propose a natural estimation procedure and experimentally illustrate the theoretical claims on low-dimensional synthetic and high-dimensional  image-based data. 
\end{itemize}

\section{Related works}

Many of the works on provable identification of representations trace their roots to non-linear ICA \citep{hyvarinen1999nonlinear}. \cite{hyvarinen2016unsupervised,hyvarinen2017nonlinear} were the first to use auxiliary information in the form of timestamps and additional structure on the latent evolution to achieve provable identification. Since then, these works have been generalized in many exciting ways. \citet{khemakhem2020variational} assume independence of latents conditional on auxiliary information, and several of these assumptions were further relaxed by \citet{khemakhem2020ice}.

Our work builds on the machinery developed we developed in 
\citet{ahuja2022properties}. There we showed that if we know the mechanisms that drive the evolution of latents, then the latents are identified up to equivariances of these mechanisms. However, we left the question of achieving exact identification without such knowledge open. Here we consider a class of mechanisms where an agent's actions impact the latents through unknown perturbations. We show how to achieve identification by exploiting the sparsity in the perturbations. This class of perturbations was first leveraged to prove identification by \citet{locatello2020weakly}. 
However, \citeauthor{locatello2020weakly} assume that the latents are independent, whereas we make no assumptions on the distribution other than continuity. Our work also connects to an insightful line of work on multi-view ICA \citep{gresele2020incomplete}. \citeauthor{gresele2020incomplete} assume independence of latents and prove identification under multiple views of the same latent through multiple decoders.

\citet{klindt2021towards} and \citet{lachapelle2022disentanglement} exploit different forms of sparsity in time-series settings to attain identification. Both works require assumptions on the parametric form of the latents (e.g., Laplacian, conditional exponential), auxiliary information observed (e.g., actions, timestamp), and the structure of the graphical model dictating the interactions between the latents and auxiliary information to arrive at identification.   \citet{yao2021learning} and \citet{lippe2022citris} model the latent evolution as a structural causal model unrolled in time. \citeauthor{yao2021learning} exploit non-stationarity and sufficient variability dictated by the auxiliary information to provide identification guarantees. \citeauthor{lippe2022citris} exploit causal interventions on the latents to provide identification guarantees but require the knowledge of intervention targets and assume an invariant causal model describing the relations between any adjacent time frames. In concurrent work, \citet{brehmer2022weakly} leverage data generated under causal interventions as a source of weak supervision and prove identification for structural causal models that are diffeomorphic transforms of exogenous noise.  In addition to the above, there are a number of recent papers that explain the success of self-supervised contrastive learning through the lens of identification of representations.  \citet{zimmermann2021contrastive} showed that encoders minimizing contrastive losses identify the latents generated from distributions such as the von Mises-Fisher distribution. 
\citet{von2021self} depart from the distributional assumptions made by \citet{zimmermann2021contrastive} and show that data augmentations filter out ``nuisance'' from the semantically relevant content to achieve blockwise identification.

\section{Latent identification under sparse perturbations}
\label{sec:theory}

\paragraph{Data Generation Process}
We start by describing the data generation process used for the rest of the work. 
There are two classes of variables we consider -- a) unobserved latent variables $Z \in \mathcal{Z} \subseteq \mathbb{R}^{d}$ and b) observed variables $X \in \mathcal{X} \subseteq \mathbb{R}^{n}$.  The latent variables $Z$ are sampled from a distribution $\mathbb{P}_Z$ and then transformed by a map $g:\mathbb{R}^{d} \rightarrow \mathbb{R}^{n}$, where $g$ is injective and analytic\footnote{A \emph{analytic} function, $g$, is an infinitely differentiable function such that for all $z'$ in its domain, the Taylor series evaluated at $z'$ converges pointwise to $g(z')$}, to generate $X$. We write this as follows
\begin{equation}
z \sim \mathbb{P}_Z  \qquad x \leftarrow g(z)
\label{eqn1}
\end{equation}
where $z$ and $x$ are realizations of the random variables $Z$ and $X$ respectively. It is impossible to invert $g$ just from the realizations of $X$ \citep{hyvarinen1999nonlinear, locatello2019challenging}. Most  work has gone into understanding how structure of latents $Z$ and auxiliary information (e.g., timestamps, weak labels) play a role in solving the above problem. In this work, we depart from these assumptions and instead investigate the role of data generated under perturbations of latents to achieve identification.  Define the set of perturbations as $\mathcal{I} = \{1, \cdots, m\}$ and the corresponding perturbation vectors as $\Delta = \{\delta_1, \cdots, \delta_m\}$, where $\delta_i$ is the $i^{th}$ perturbation. Each latent $z$ is sampled from an arbitrary and unknown distribution $\mathbb{P}_Z$. The {\em same set of unknown perturbations} in $\Delta$ are applied to each $z$ to generate $m$  perturbed latents $\{ \tilde{z}_k \}_{k=1}^{m}$ per sampled $z$ and the corresponding observed vectors $\{\tilde{x}_k\}_{k=1}^{m}$.   Each of these latents are transformed by the map $g$ and we observe $(x, \tilde{x}_1, \cdots, \tilde{x}_m)$. Our goal is to use these observations and estimate the underlying latents. We summarize this data generation process (DGP) in the following assumption.

\begin{assumption}
	\label{assum: dgp}
	The DGP follows
	\begin{equation}
	z \sim \mathbb{P}_Z, x\leftarrow g(z) \qquad    \tilde{z}_k \leftarrow z + \delta_k, \forall k \in \mathcal{I} \qquad\tilde{x}_k \leftarrow g(\tilde{z}_k),  \forall k \in \mathcal{I}
	\label{eqn_perturb_obs}
	\end{equation}
	where $g$ is injective and analytic, and $Z$ is a continuous random vector with full support over $\mathbb{R}^{d}$. \footnote{The assumption on the support of $Z$ can be relaxed.}
\end{assumption}


To better understand the above DGP, let us turn to some examples. Consider a setting where an agent is interacting with an environment containing several balls (See Figure \ref{fig: ball_env}). The latent $z$ captures the properties of the objects; for example, in Figure \ref{fig: ball_env}, $z$ just captures the positions of each ball, but in general it could include more properties such as velocity, shape, color, etc.. The agent perturbs the objects in the scene by $\delta_k$, which can modify a single property associated with one object or multiple properties from one or more objects depending on how the agent acts. Note that when the agent perturbs a latent, it can lead to downstream effects. For instance, if the agent moves a ball to the edge of the table, the ball falls in subsequent frames. For this work, we only consider the observations just before and after the perturbation and not the downstream effects. In the Appendix, we explain these downstream effects using structural causal models (See Section 7.2). We also explain the connection between the perturbations in equation \eqref{eqn_perturb_obs} (based on \cite{locatello2020weakly}) and causal interventions.  The above example is typical of a reinforcement learning environment, other examples include natural videos with sparse changes (e.g., MPI3D data \citep{gondal2019transfer}).

In the above DGP in equation \eqref{eqn_perturb_obs}, we assumed that for each scene $x$ there are multiple perturbations. It is possible to extend our results to settings where we perturb each scene only once, given a sufficiently diverse set of perturbations, i.e., for a small neighborhood of a scene around $x$, each scene in the neighbourhood receives a different perturbation. We compare these two approaches experimentally.

\paragraph{Learning objective} 
The learner's objective is to use the observed samples $(x, \tilde{x}_1, \cdots, \tilde{x}_m)$ generated by the DGP in Assumption \ref{assum: dgp} and learn an encoder $f:\mathbb{R}^{n} \rightarrow \mathbb{R}^{d}$ that inverts the function $g$ and recovers the true latents.  For each observed sample $(x, \tilde{x}_1, \cdots, \tilde{x}_m)$, the learner compares all the pairs $(x, \tilde{x}_k)$ pre- and post-perturbation.  For every unknown perturbation $\delta_k$ used in the DGP in equation \eqref{eqn_perturb_obs}, the learner guesses the perturbation $\delta_k^{'} $ and enforces that the latents predicted by the encoder for $x$ and $\tilde{x}_k$ are consistent with the guess. We write this as $\forall\, (x, \tilde{x}_1, \cdots, \tilde{x}_m) $  generated by DGP in \eqref{eqn_perturb_obs} 
\begin{equation}
f(\tilde{x}_k)= f(x)  + \delta_k^{'}.
\label{eqn:intv_identity}
\end{equation}
We denote the set of guessed perturbations as $\Delta^{'} = \{\delta_1^{'}, \cdots, \delta_m^{'}\}$, where $\delta_i^{'}$ is the guess for perturbation $\delta_i$.   We can turn the above identity into a mean square error loss given as 

\begin{equation}
\min_{f, \Delta^{'}} \mathbb{E} \Big[ \Big\| f(\tilde{x}_k) - f(x)  - \delta_k^{'}\Big\|^2 \Big]
\label{eqn:mse}
\end{equation}

where the expectation is taken over observed samples generated by the DGP in \eqref{eqn_perturb_obs} and the minimization is over all the possible maps $f$ and perturbation guesses in the set $\Delta^{'}$.  Note that a trivial solution to the above problem is an encoder that maps everything to zero, and all guesses equal zero. In the next section, we get rid of these trivial solutions by imposing an additional requirement that the span of the set $\Delta^{'}$ is $\mathbb{R}^{d}$. It is worth pointing out that we do not restrict the set of $f$'s to injective maps in theory and experiments.  We denote the latent estimated by the encoder for a point $x$ as $\hat{z} = f(x)$. It is related to the true latent as follows $\hat{z} = f \circ g (z) = a(z)$,  where $a$ is some function that relates true $z$ to estimated $\hat{z}$. In the next section, we show that if perturbations are diverse, then $a$ is an affine transform. Further, we show that if perturbations are sparse, then $a$ takes an even simpler form.

\subsection{Sparse perturbations}

We first show that it is possible to identify the true latents up to an affine transformation without any sparsity assumptions. Later, we leverage sparsity to strengthen identification guarantees.

\begin{assumption} \label{assm:span} The dimension of the span of the perturbations in equation $\eqref{eqn_perturb_obs}$ is $d$, i.e., $\mathsf{dim}\Big(\mathsf{span}\big(\Delta\big)\Big)  = d$. 
\end{assumption}
The above assumption implies that the perturbations are diverse. We now state a regularity condition on the function $a$.

\begin{assumption}
	\label{assum: analytic_measure}
	$a$ is an analytic function. For each component $i\in \{1,\cdots, d\}$ of $a(z)$ and each component $j\in \{1, \cdots, d\}$ of $z$, define the set $\mathcal{S}^{ij} = \{\theta \; |\; \nabla_{j} a_i(z+b) = \nabla_{j} a_i(z) + \nabla^2_{j} a_{i}(\theta) b, z\in \mathbb{R}^{d}\} $, where $b$ is  a fixed vector in $\mathbb{R}^{d}$. Each set  $\mathcal{S}^{ij}$ has a non-zero Lebesgue measure in $\mathbb{R}^{d}$. 
\end{assumption}

If we restrict the encoder $f$ to be analytic, then $a$ is analytic since $g$ is also analytic, thus satisfying the first part of the above assumption. The second part of the above assumption can be understood as follows: 
suppose we have a scalar valued function $h:\mathbb{R}\rightarrow \mathbb{R}$ that is differentiable. If we expand $h(u+v)$ around $h(u)$, by the mean value theorem we get $h(u+v) = h(u) + h'(\epsilon) v$, where $\epsilon \in [u, u+v]$. If we vary $u$ to take all the values in $\mathbb{R}$, then $\epsilon$ also varies. The above assumption states that the set of $\epsilon's$ has a non-zero Lebesgue measure.  Under the above assumptions, we show that an encoder that solves equation \eqref{eqn:intv_identity}  identifies true latents up to an affine transform, i.e., $\hat{z} = Az+c$,  where $A \in \mathbb{R}^{d\times d}$ is a matrix  and $c \in \mathbb{R}^d$ is an offset.

\begin{proposition}
	\label{theorem: additive_intervention}
	If Assumptions \ref{assum: dgp}, \ref{assm:span}, and \ref{assum: analytic_measure} hold, then the encoder that solves  equation \eqref{eqn:intv_identity} (with $\Delta^{'}$ s.t. $\mathsf{dim}\Big(\mathsf{span}\big(\Delta^{'}\big)\Big)  = d$) identifies true latents up to an invertible affine transform, i.e. $\hat{z} = A z + c$, where $A \in \mathbb{R}^{d\times d}$ is an invertible matrix  and $c \in \mathbb{R}^d$ is an offset.
\end{proposition}

The proof of above proposition follows the proof technique from \cite{ahuja2022properties}, for further details refer to the Appendix (Section 7.1). We interpret the above result in the context of the agent interacting with balls (as shown in Figure \ref{fig: ball_env}), where the latent vector $z$ captures the $x$ and $y$ coordinates of the $n_{\mathsf{balls}}$. Under each perturbation, the balls move along the vector dictated by the perturbation. If there are at least $2n_{\mathsf{balls}}$  perturbations, then the latents estimated by the learned encoder are guaranteed to be an affine transformation of the actual positions of the balls. 

\subsubsection{Non-overlapping perturbations}

In Proposition \ref{theorem: additive_intervention}, we showed affine identification guarantees for the DGP from Assumption \ref{assum: dgp}.  We now explore identification when perturbations are one-sparse, i.e., one latent changes at a time.

\begin{assumption}
	\label{assm: sparse1}
	The perturbations in $\Delta$ are one-sparse, i.e., each $\delta_i \in \Delta$ has one non-zero component. 
\end{assumption}

Next, we show that under one-sparse perturbations, the latents estimated identify true latents up to permutation and scaling.

\begin{theorem}
	\label{theorem: indiv_level_intervention}
	If Assumptions \ref{assum: dgp}-\ref{assm: sparse1} hold and the number of perturbations per example equals the latent dimension, $m=d$, \footnote{We can relax this condition to $m\geq d$, refer to the Appendix (Section 7.2) for details.} then the encoder that solves  equation \eqref{eqn:intv_identity} (with $\Delta^{'}$ as one-sparse and $\mathsf{dim}\Big(\mathsf{span}\big(\Delta^{'}\big)\Big)  = d$) identifies true latents up to permutation and scaling, i.e. $\hat{z} = \Pi \Lambda z + c$, where $\Lambda \in \mathbb{R}^{d\times d}$ is an invertible diagonal matrix, $\Pi \in \mathbb{R}^{d\times d}$ is a permutation matrix and $c$ is an offset.
\end{theorem}

For the proof of above theorem, refer to Section 7.1 in the Appendix. The theorem does not require that learner knows either the identity or amount each component changed. However, the learner has to use one-sparse perturbations as guesses. Suppose the learner does not know that the actual perturbations are one-sparse and instead uses guesses that are $p$-sparse, i.e., $p$ latents change at one time. In such a case, the $\hat{z}$ and true $z$ are related to each other through a permutation and block diagonal matrix, i.e., we can replace $\Lambda$ in the above result to be a block diagonal matrix instead of a diagonal matrix (see Section 7.2 in the Appendix for details). In the context of the ball agent interaction environment from Figure \ref{fig: ball_env}, the above result states that provided the agent interacts with one coordinate of each ball at a time, it is possible to learn the position of each ball up to scaling errors.  

We now consider a natural extension of the setting above, where the perturbations simultaneously operate on blocks of latents. In the ball agent interaction environment, this can lead to multiple scenarios -- i) the agent interacts with one ball at a time but perturbs both coordinates simultaneously, ii) the agent interacts with several balls simultaneously. 

Consider a perturbation $\delta_i \in \Delta$ (from equation \eqref{eqn_perturb_obs}).  We define the block of latents that is impacted under perturbation $\delta_i\in \Delta$ as $ \{j\;| \;\delta_i^{j}\not=0, j\in \{1, \cdots, d\}\}$, where $\delta_i^j$ is the $j^{th}$ component of $\delta_i$. We group the perturbations in $\mathcal{I}$ based on the block they act upon, i.e. perturbations in the same group act on the same block of latents. Define the set of the groups corresponding to perturbations in $\mathcal{I}$ as  $\mathcal{G}_{\mathcal{I}}$. Define the set of corresponding blocks as $\mathcal{B}_{\mathcal{I}} = \{\mathcal{B}_1, \cdots, \mathcal{B}_g\}$, where $\mathcal{B}_k$ is the block impacted by perturbations in group $k$.  If $\mathcal{B}_{\mathcal{I}} $ partitions the set of latent components indexed $\{1, \cdots, d\}$, then it implies all the distinct blocks are non-overlapping. We formally define this below. 


\begin{definition}
	\label{blockwise_defn}
	\textbf{Blockwise and non-overlapping perturbations.}  If the the set of blocks $\mathcal{B}_{\mathcal{I}}$ corresponding to perturbations $\mathcal{I}$ form a partition of $\{1, \cdots, d\}$, then $\mathcal{I}$ is said to be blockwise and non-overlapping. Formally stated,   any two distinct $\mathcal{B}_i, \mathcal{B}_j \in \mathcal{B}_{\mathcal{I}}$ do not intersect, i.e.,  $\mathcal{B}_i \cap \mathcal{B}_j = \emptyset$, and $\cup_i \mathcal{B}_i = \{1,\cdots, d\}$. 
\end{definition}

From the above definition it follows that two perturbations either act on the same block or completely different blocks with no overlapping variables.

\begin{assumption}
	\label{assum:block_level_intervention}
	The perturbations $\mathcal{I}$ (used in equation \eqref{eqn_perturb_obs}) are blockwise and non-overlapping (see Definition \ref{blockwise_defn}). Each  perturbation in $\mathcal{I}$ is $p$-sparse, i.e., it impacts blocks of length  $p$ ($p\leq d$) at a time.
\end{assumption}

\begin{assumption}
	\label{assum: span_agent}
	The learner knows the group label for each perturbation $i \in \mathcal{I}$.  Therefore, any two perturbations in $\Delta^{'}$ associated with same group in $\mathcal{G}_{\mathcal{I}}$ impact the same block of latents.
\end{assumption}

We illustrate the above Assumptions \ref{assum:block_level_intervention},  \ref{assum: span_agent} in the following example. Consider the ball agent interaction environment (Figure \ref{fig: ball_env}). $z = [z_{1x}, z_{1y}, \cdots, z_{n_{\mathsf{balls}}x},z_{n_{\mathsf{balls}}y}]$ is the vector of positions of all balls, where $z_{ix/y}$ is the $x/y$ coordinate of ball $i$. If the agent randomly perturbs ball $i$, then it changes the block $(z_{ix}, z_{iy})$. We would call such a system $2$-sparse. All the perturbations on ball $i$ are in one group. Since the agent knows the group of the perturbation, it does not know the ball index but it knows whenever we interact with the same ball.

\begin{definition}
	If the latent variables recovered $\hat{z}= \Pi \tilde{\Lambda}  z + c$, where $\Pi$ is a permutation matrix and $\tilde{\Lambda}$ is a block-diagonal matrix, then the latent variables are said to be recovered up to permutations and block-diagonal transforms. 
\end{definition}

In the theorem that follows, we show that under the assumptions made in this section, we achieve identification up to permutations and block-diagonal transforms with invertible $p\times p$ blocks.

\begin{theorem}
	\label{theorem: blockwise_intervention}
	If Assumptions \ref{assum: dgp}-\ref{assum: analytic_measure}, \ref{assum:block_level_intervention}, \ref{assum: span_agent} hold,  then the encoder that solves  equation \eqref{eqn:intv_identity} (where $\Delta^{'}$ is $p$-sparse, $\mathsf{dim}\Big(\mathsf{span}\big(\Delta^{'}\big)\Big)  = d$) identifies true latents up to permutation and block-diagonal transforms, i.e. $f(x)=\hat{z} = \Pi\tilde{\Lambda} z + c$, where $\tilde{\Lambda} \in \mathbb{R}^{d\times d}$ is an invertible block-diagonal matrix with blocks of size $p\times p$, $\Pi \in \mathbb{R}^{d\times d}$ is a permutation matrix and $c \in \mathbb{R}^d$ is an offset.
\end{theorem}

For the proof of the above theorem, refer to Section 7.1 in the Appendix. From the above theorem, we gather that the learner can separate the perturbed blocks. However, the latent dimensions within the block are linearly entangled. In the ball agent interaction with $2$-sparse perturbations, the above theorem implies that the agent can separate each ball out but not their respective $x$ and $y$ coordinates. In the above theorem, we require the learner to know the group of each intervention (Assumption \ref{assum: span_agent}). In Section 7.2 in the Appendix, we discuss ideas on how to relax this assumption.

\subsubsection{Overlapping perturbations}

In the previous section, we assumed that the blocks across different perturbations are non-overlapping. This section relaxes this assumption and allows the perturbation blocks to overlap. We start with a motivating example to show how overlapping perturbations can lead to stronger identification. 

Consider the agent interacting with two balls, where $z = [z_{1x}, z_{1y}, z_{2x}, z_{2y}]$ describes the coordinates of the two balls.  The agent perturbs the first ball and then perturbs the second ball. For the purpose of this example, assume that these perturbations satisfy the assumptions in Theorem \ref{theorem: blockwise_intervention}. We obtain that the estimated position of each ball $\hat{z}_{ix/y}$ is linearly entangled w.r.t the true $x$ and $y$ coordinates. For the first ball we get
$\hat{z}_{1x} = a_1 z_{1x} + a_2z_{1y} +a_3$.
We also have the agent perturb the $x$ coordinates of the first and second ball together and then it does the same with the $y$ coordinates.  We apply Theorem \ref{theorem: blockwise_intervention} and obtain that the estimated $x$ coordinates of each ball are linearly entangled. We write this as $\hat{z}_{1x} = b_1z_{1x}  + b_2z_{2x} + b_3$. We take a difference of the two relations for $\hat{z}_{1x}$ to get 
\begin{equation}
(a_1-b_1)z_{1x} + a_2z_{1y} -b_2 z_{2x}  +a_3-b_3 = 0
\end{equation}

Since the above has to hold for all $z_{1x},z_{1y}, z_{2x}$, we get $a_1=b_1$, $a_2=0$, $b_2=0$ and $a_3=b_3$. Thus $\hat{z}_{1x} = a_1z_{1x}+a_3$. Similarly, we can disentangle the rest of the balls. 

We take the insights from the above example and generalize them below. Let us suppose that from the set of perturbations $\mathcal{I}$ we can construct at least two distinct subsets $\mathcal{I}_1$ and $\mathcal{I}_2$ such that both subsets form a blockwise non-overlapping perturbation (see Definition \ref{blockwise_defn}).  Perturbations in $\mathcal{I}_1$ ($\mathcal{I}_2$) partition $\{1, \cdots, d\}$ into blocks $\mathcal{B}_{\mathcal{I}_1}$ ($\mathcal{B}_{\mathcal{I}_2}$) respectively. 
It follows that there exists at least two blocks  $\mathcal{B}^{1} \in \mathcal{B}_{\mathcal{I}_1}$ and $\mathcal{B}^{2} \in \mathcal{B}_{\mathcal{I}_2}$ such that $\mathcal{B}^{1} \cap \mathcal{B}^{2} \not= \emptyset$. From Theorem \ref{theorem: blockwise_intervention}, we know that we can identify latents in block $\mathcal{B}^{1}$ and $\mathcal{B}^{2}$ up to affine transforms. In the next theorem, we show that we can identify latents in each of the blocks $\mathcal{B}^{1} \cap \mathcal{B}^{2}$, $\mathcal{B}^{1} \setminus \mathcal{B}^{2}$, $\mathcal{B}^{2} \setminus \mathcal{B}^{1}$ up to affine transforms.

\begin{assumption}
	\label{assum1:block_level_intervention}
	Each perturbation in $\mathcal{I}$ is $p$-sparse.  The perturbations in each group span a $p$-dimensional space, i.e., $\forall q \in \mathcal{G}_{\mathcal{I}}, \; \mathsf{dim}\Big(\mathsf{span}\Big(\{\delta_i\}_{i\in q} \Big)\Big) =p$.  There exist at least two distinct subsets of perturbations $\mathcal{I}_1\subseteq \mathcal{I}$ and $\mathcal{I}_2\subseteq \mathcal{I}$ that are both blockwise and non-overlapping.
\end{assumption}

\begin{theorem}
	\label{theorem: intersecT}
	Suppose Assumptions \ref{assum: dgp}, \ref{assum: analytic_measure}, \ref{assum: span_agent} and \ref{assum1:block_level_intervention} hold. Consider the subsets $\mathcal{I}_1$ and $\mathcal{I}_2$ that satisfy Assumption \ref{assum1:block_level_intervention}. For every pair of blocks,  $\mathcal{B}^{1} \in \mathcal{B}_{\mathcal{I}_1}$ and $\mathcal{B}^{2} \in \mathcal{B}_{\mathcal{I}_2}$, the encoder that solves equation \eqref{eqn:intv_identity} (where $\Delta^{'}$ is $p$-sparse, $\mathsf{dim}\Big(\mathsf{span}\big(\Delta^{'}\big)\Big)  = d$) identifies latents in each of the blocks $\mathcal{B}^{1} \cap \mathcal{B}^{2}$, $\mathcal{B}^{1} \setminus \mathcal{B}^{2}$, $\mathcal{B}^{2} \setminus \mathcal{B}^{1}$ up to invertible affine transforms.
\end{theorem}

For the proof of the above theorem, refer to Section 7.1 in the Appendix. From the above theorem, it follows that if blocks overlap at one latent only, then all such latents are identified up to permutation and scaling. We now construct an example to show the identification of all the latents under overlapping perturbations. Suppose we have a $4$ dimensional latent. The set of all contiguous blocks of length $2$ is given as follows $\{\{1,2\}, \{2,3\}, \{3, 4\}, \{4,1\}\}$. Different $2$-sparse perturbations impact these blocks. Observe that every component between $1$ to $4$ gets to be the first element of a block exactly once and the last element of the block exactly once. As a result, each latent gets to be the only element at the intersection of two blocks. We apply Theorem \ref{theorem: intersecT} to this case and get that all the latents are identified up to permutation and scaling. We generalize this example below.



\begin{assumption}
	\label{assum:block_level2}
	$\mathcal{B}_{\mathcal{I}}$ is a set of all the contiguous blocks of length $p$, where $p<d$. The perturbations in each block span a $p$ dimensional space. Further, assume that $d\; \mathsf{mod} \;p =0$. 
\end{assumption} 

In the above assumption, we construct $d$ contiguous blocks such that a blocks of length $p$. The construction ensures that each index in $\{1, \cdots, d\}$ forms the first element of exactly one block and last element of exactly one block.   In the next theorem, we show that under the above assumption,  we achieve identification up to permutation and scaling. 

\begin{theorem}
	\label{cor:block_level_intervention}
	Suppose Assumptions \ref{assum: dgp}, \ref{assum: analytic_measure}, \ref{assum: span_agent} and \ref{assum:block_level2} hold, then the encoder that solves the identity in equation \eqref{eqn:intv_identity} (where $\Delta^{'}$ is $p$-sparse, $\mathsf{dim}\Big(\mathsf{span}\big(\Delta^{'}\big)\Big)  = d$) identifies true latents up to permutations and scaling, i.e., $\hat{z} = \Pi \Lambda z +c$, where $\Pi \in \mathbb{R}^{d\times d}$ matrix and $\Lambda \in \mathbb{R}^{d\times d}$ is a diagonal matrix.
\end{theorem}
\begin{table}
	\captionsetup{width=.75\textwidth}
	\caption{Comparing MCC and BMCC for non-overlapping perturbations. The number of perturbations applied for each example is given in parenthesis}
	\label{table_synth}
	\renewcommand{\arraystretch}{1.2}
	\centering
	\begin{tabular}{llll|ll }
		\toprule
		$d$  & $p_Z$   & MCC     & MCC                   & BMCC    & BMCC        \\ 
		&         & C-wise ($d$)  & C-wise  ($1$)            & B-wise ($d$) & B-wise ($1$)\\ \hline 
		$6$  & Normal  & $0.99 \pm 0.00$ &  {\color{gray}$0.99 \pm 0.00$}   & $0.99 \pm 0.00$ &{\color{gray}$0.99 \pm 0.01$}\\
		$10$ & Normal  & $0.99 \pm 0.00$ & {\color{gray}$0.99 \pm 0.01$}    & $0.99 \pm 0.00$    &{\color{gray}$0.91 \pm 0.02$}  \\
		$20$ & Normal  & $0.99 \pm 0.00$ & {\color{gray}$0.88 \pm 0.03$}      & $0.99\pm 0.00$ &{\color{gray}$0.90 \pm 0.01$}\\ \hline
		$6$  & Uniform & $0.99 \pm 0.00$ &  {\color{gray}$0.99 \pm 0.00$}    &$0.99 \pm 0.00 $  &{\color{gray}$0.96 \pm 0.04$}\\
		$10$ & Uniform & $0.99 \pm 0.00$ &  {\color{gray}$0.99 \pm 0.01$}  & $0.99 \pm 0.00 $  & {\color{gray}$0.81 \pm 0.05$} \\ 
		$20$ & Uniform & $0.99 \pm 0.00$ &  {\color{gray}$0.82 \pm 0.02$}    & $0.85 \pm 0.08$ & {\color{gray}$0.51 \pm 0.04$} \\  
		\bottomrule
	\end{tabular}
\end{table}

For the proof of the above theorem, refer to Section 7.1 in the Appendix. The total number of perturbations required in the above theorem is $p\times d$. If we plug $p=1$, we recover Theorem \ref{theorem: indiv_level_intervention} as a special case. The above result highlights that if the block lengths are larger, then we need to scale the number of perturbations accordingly by the same factor to achieve identification up to permutation and scaling. We assumed a special class of perturbations operating on contiguous blocks. In general, the total number of distinct blocks can be up to $d \choose p$. Suppose $s$ distinct random blocks of length $p$ are selected for perturbations. As $s$ grows, we reach a point where each latent component is at the intersection of two blocks from different sets of blockwise non-overlapping perturbations. At that point, we identify all latents up to permutation and scaling.

\section{Experiments}
\label{sec: expmts}


\paragraph{Data generation processes}  We conducted two sets of experiments -- low-dimensional synthetic and high-dimensional image-based inputs -- that follow the DGP in equation \eqref{eqn_perturb_obs}. In the low-dimensional synthetic experiments we experimented with two choices for $\mathbb{P}_Z$ a) uniform distribution with independent latents, b) normal distribution with latents that are blockwise independent (with block length $d/2$). We used an invertible multi-layer perceptron (MLP) (with $2$ hidden layers) from \cite{zimmermann2021contrastive} for $g$. We evaluated for latent dimensions $d\in\{6,10,20\}$. The training and test data size was $10000$ and  $5000$ respectively. For the image-based experiments we used PyGame \citep{pygame}'s rendering engine for $g$ and generated $64\times 64$ pixel images that look like those shown in Figure \ref{fig: ball_env}. The coordinates of each ball, $z_i$, were drawn independently from a uniform distribution, $z_i \sim \mathcal{U}(0.1, 0.9)$. We varied the number of balls from $2$ ($d=4$) to 4 ($d=8$). For these experiments, there was no fixed-size training set; instead the images are generated online and we trained to convergence. Because these problems are high dimensional, we only sampled a single perturbation for each image.

\begin{table}
	\begin{minipage}{0.6\linewidth}
		\centering
		\includegraphics[width=3in]{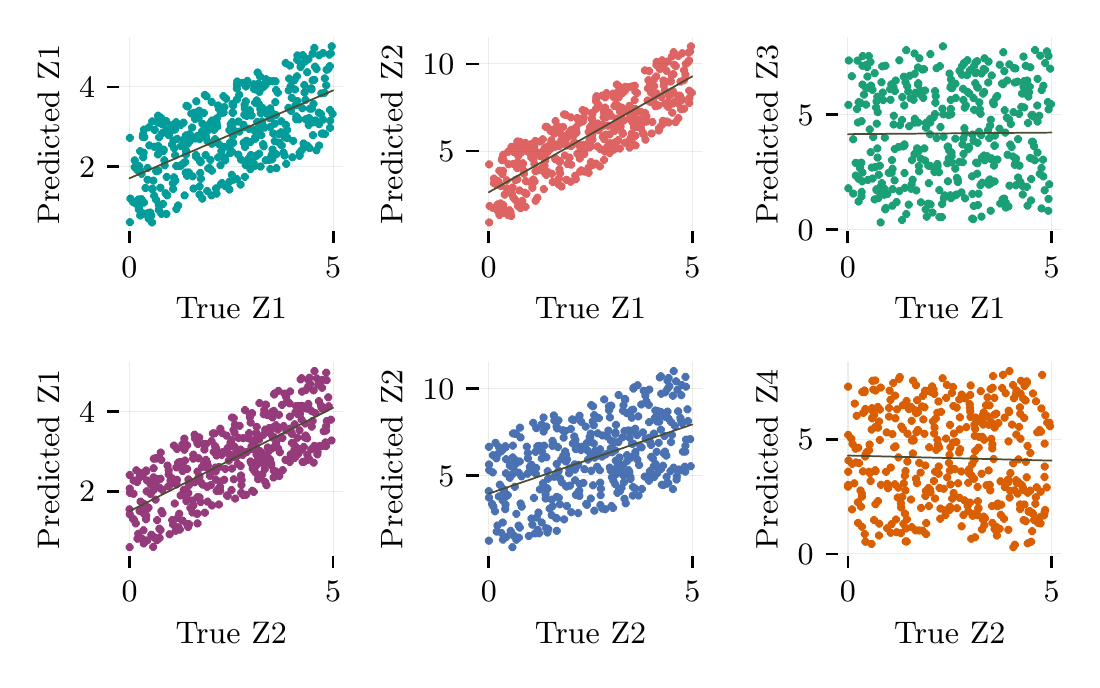}
		\captionof{figure}{Illustrating blockwise dependence ($d=10$).}
		\label{fig: com_reg_noverlap}
	\end{minipage}
	\begin{minipage}{0.37\linewidth}
		\renewcommand{\arraystretch}{1.2}
		\caption{ MCC for B-wise (overlap).}
		\label{table_synth2}
		\centering
		\begin{tabular}{lll}
			\toprule
			$d$    & Distribution   & MCC\\ \hline 
			$6$ & Normal  & $0.95 \pm 0.01$    \\
			$10$     & Normal  & $0.96 \pm 0.01$    \\ 
			$20$   & Normal  &  $0.99\pm 0.01$\\ \hline
			$6$ & Uniform & $0.86 \pm 0.03$     \\
			$10$     & Uniform & $0.88\pm 0.03$      \\ 
			$20$    & Uniform &  $0.81\pm 0.03$ \\  
			\bottomrule
		\end{tabular}
	\end{minipage}\hfill
\end{table}

\paragraph{Loss function, architecture, evaluation metrics}
In all the experiments we optimized equation \eqref{eqn:mse} with square error loss. The encoder $f$ was an MLP with two hidden layers of size $100$ for the low-dimensional synthetic experiments and a ResNet-18 \citep{He2015} for the image-based experiments.  Further training details such as the optimizers used, hyperparameters etc. are in the Appendix (Section 7.3). We used the mean correlation coefficient (MCC) \citep{hyvarinen2016unsupervised} to verify the claims in Theorems \ref{theorem: indiv_level_intervention} and \ref{cor:block_level_intervention}. If MCC equals one, then the estimated latents identify true latents up to permutation and scaling.  We extend MCC to blockwise MCC (BMCC) to verify the claims in Theorem \ref{theorem: blockwise_intervention}. If BMCC equals one, then the estimated latents identify true latents up to permutation and block-diagonal transforms. Further details 
are in the Appendix (Section 7.3). The codes to reproduce these experiments can be found at \url{https://github.com/ahujak/WSRL}.

\begin{table}[t]
	\begin{minipage}{0.37\linewidth}
		\caption{Image experiments}
		\label{table_img}
		\centering
		\begin{tabular}{llll}
			\toprule
			$d$       & MCC & MCC& MCC\\ 
			&C-wise& B-wise   & B-wise  \\
			& & $d$  & ${d \choose p}$   \\  \hline
			$4$            & $0.994$ &$0.710$ &$0.864$ \\  
			$6$          & $0.981$ &$ 0.817$&$0.912$ \\  
			$8$         & $0.975$ & $0.866$&$0.934 $\\  
			\bottomrule
		\end{tabular}
	\end{minipage}%
	\begin{minipage}{0.62\linewidth}
		\centering
		\includegraphics[width=0.16\linewidth, trim=0 0 0 0.cm]{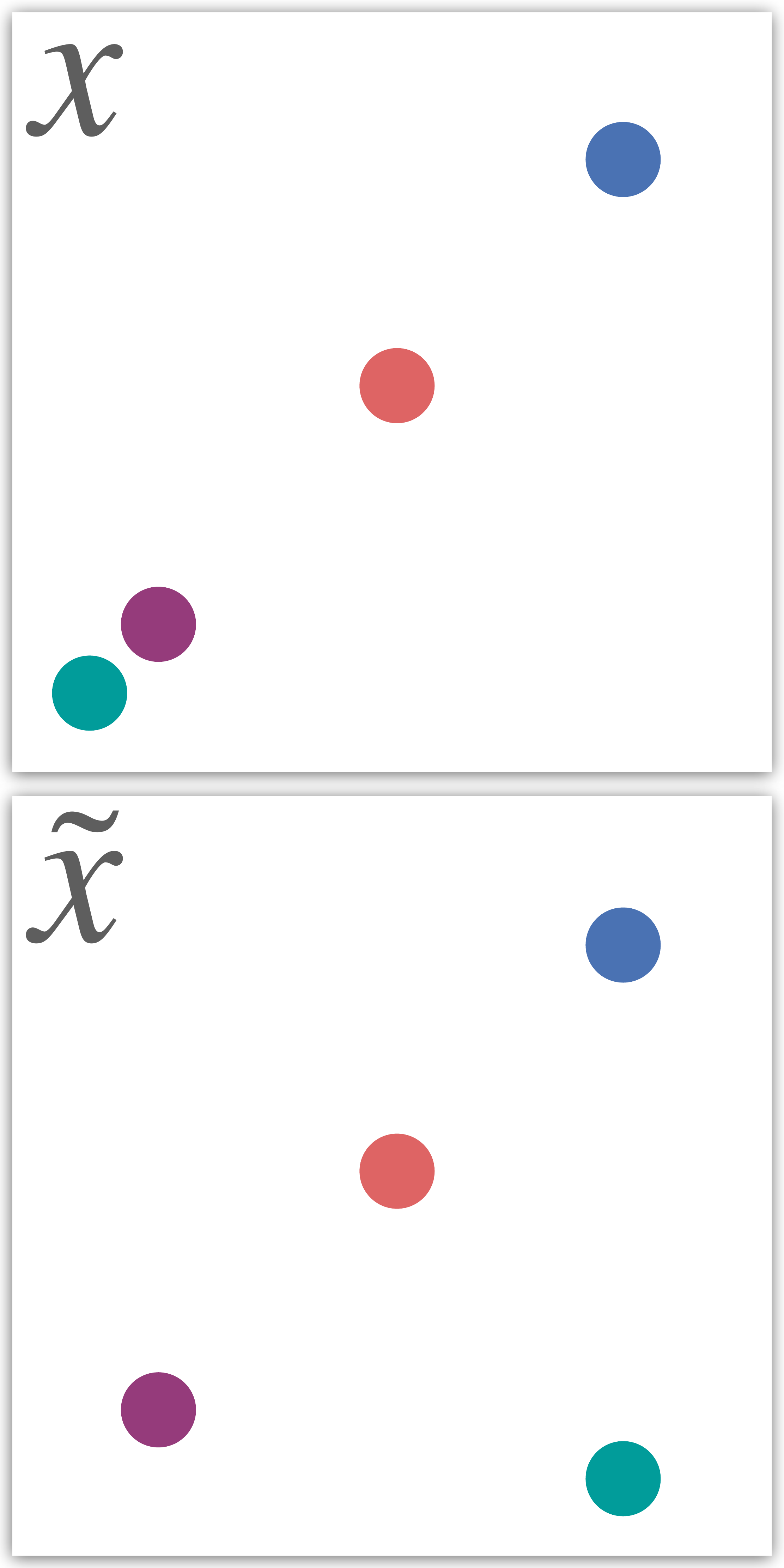}
		\vspace{0.1cm}
		\includegraphics{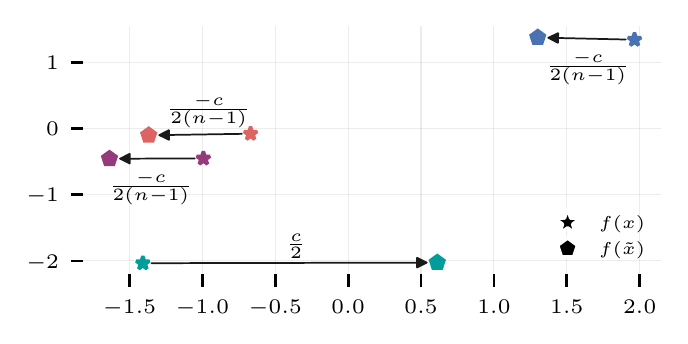}
	\end{minipage}
	\captionof{figure}{\emph{(Left)} Results for the image-base experiments. \emph{(Centre)} Example images in which the bottom left ball is shifted to the right. \emph{(Right)} A trained encoder's predictions for the two images shown in \emph{(centre)}. The green ball prediction shifts right by $\approx\frac{c}{2}$ and the other balls left by $\approx\frac{c}{2(n-1)}$. For further illustrations, refer to the animations in \url{https://github.com/ahujak/WSRL}.}
	\label{fig: images}
\end{table}

\paragraph{Non-overlapping perturbations}
\label{para: non-ovlap}
We start with results from experiments with one-sparse perturbations. The set $\Delta$ consists of $m=d$ one-sparse perturbations that span a $d$ dimensional space. In the context of the image experiments, these perturbations correspond to moving each ball individually along a single axis. The learner solves the identity in equation \eqref{eqn:intv_identity} using a set of random one-sparse perturbations $\Delta^{'}$  that span a $d$ dimensional space. 
In Table \ref{table_synth}, we used the low-dimensional synthetic data generating process to compare the effect of (i) applying all $m=d$ perturbations to each instance $z$ (following the DGP in \eqref{eqn_perturb_obs}), against a more practical setting (ii) where a perturbation is selected uniform at random from $\Delta$ and applied to each instance $z$.  The results for (i) are shown in black and the results for (ii) are shown in gray font in the {\em C-wise} (componentwise) column in Table  \ref{table_synth}. We observed high MCCs in both settings. The results were similar in the more challenging image-based experiments (see Table \ref{table_img}, C-wise column) with MCC scores $> 0.97$ for all the settings that we tested, as expected given 
the results presented in Theorem \ref{theorem: indiv_level_intervention}. 

Next, we chose the set of perturbations $\Delta$ to comprise of $d$ $2$-sparse non-overlapping perturbations that span a $d$ dimensional space. We repeated the same synthetic experiments as above with one and $d$ perturbations per instance.   Under these assumptions we should expect to see that  pairs of latents are separated blockwise but linearly entangled within the blocks (c.f. Theorem \ref{theorem: blockwise_intervention}). We found this to be the case. The high BMCC numbers in Table \ref{table_synth} displayed under {\em B-wise}  (blockwise) column (except for $d=20$ and one perturbation per sample) show disentanglement between the blocks of latents. In Figure \ref{fig: com_reg_noverlap}, the first two rows and columns show how the predicted latents corresponding to a block are correlated with their true counterpart (see Predicted $Z_i$ vs True $Z_i$) and the other latent in the block (Predicted $Z_1$ vs True $Z_2$ and vice versa). The plots in the last column show that the predicted latents did not bear a correlation with a randomly selected latent from outside the block.

\paragraph{Overlapping perturbations}
In this section, we experimented with blocks of size two that overlap in order to conform with the setting described in Theorem \ref{cor:block_level_intervention}.   
We used the same distributions as before and only changed the type of perturbations.  The low-dimensional synthetic results are summarized in Table \ref{table_synth2}. The results were largely as expected, with a strong correspondence between the predicted and true latents reflected by high MCC values.

On the image datasets (see Table \ref{table_img}), we found that the MCC scores depended on both the number of balls and how the blocks were selected. We compared two strategies for selecting blocks of latents to perturb: either select uniformly from all adjacent pairs $\mathcal{I} = \{(i \text{ mod }d,i+1\text{ mod } d)\}$ ($d$ blocks),  or uniformly from all combinations of latent indices, $\mathcal{I} = \{(i,j):i \in \{1, \dots, d\}, j>i\}$ (${d\choose 2}$ blocks). The latter lead to higher MCC scores (ranging from $0.86$ to $0.93$) as it placed more constraints on the solution space. The dependence on the number of balls is more surprising. 
To investigate the implied entanglement from the lower MCC scores, we evaluated trained encoders on images where we kept $n_{\mathsf{balls}}-1$ balls in a fixed location and moved one of the balls (see Section 7.3 in the Appendix for example images). If the coordinates were perfectly disentangled, the encoder should predict no movement for static balls. 
When the moving ball shifted by $c$ units, the predicted location of the static balls shifted by $\approx\frac{-c}{2(n_{\mathsf{balls}}-1)}$ and the moving ball shifted $\approx\frac{c}{2}$ units. We further verified this claim and ran blockwise experiments with $n_{\mathsf{balls}}=10$ balls ($d=20$) and got MCC scores of $0.930$ and $0.969$ for $d$ and ${d \choose 2}$ blocks respectively. 
In the Appendix (Section 7.3), we show that this solution is a stationary point, and we achieve a perfect MCC of one when $n_{\mathsf{balls}} = \infty$.

\section{Discussion and limitations}
\label{sec: disc_lim}

Our work presents the first systematic analysis of the role of sparsity in achieving latent identification under unknown arbitrary latent distributions. We assume that every sample (or at least every neighborhood of a sample) experiences the same set of perturbations. A natural question is how to extend our results to settings where this assumption may not hold. 
Data augmentation provides a rich source of perturbations; our results cover translations, but extending them to other forms of augmentation is an important future direction.
We followed the literature on non-linear ICA \citep{hyvarinen2019nonlinear} and made two assumptions -- i) the map $g$ that mixes latents is injective, and ii) the dimension of the latent $d$ is known. We believe future works should aim to relax these assumptions.

\newpage 

\section{Acknowledgements}

Kartik Ahuja acknowledges the support provided by the IVADO postdoctoral fellowship funding program. Jason Hartford acknowledges the support of the Natural Sciences and Engineering Research Council of Canada (NSERC).  The authors also acknowledge the funding from Recursion Pharmaceuticals and CIFAR.

\bibliographystyle{apalike}
\bibliography{CRL.bib}

\section{Appendix}
We organize the Appendix into three sections. In Section \ref{sec: proofs}, we provide the proofs to all the propositions and the theorems. In Section \ref{sec: extension}, we discuss how some of the proposed results can be extended. In Section \ref{sec : exp_supp}, we provide supplementary materials for the experiments. 

\subsection{Proofs}
\label{sec: proofs}
We restate all the propositions and the theorems below for convenience. In the proofs that follow, we use $\Delta$ ($\Delta^{'}$) to denote the set of perturbations and the matrix of perturbations interchangeably (their usage is clear from the context). We start with the proof of Proposition \ref{theorem: additive_intervention}, which follows the proof technique from \cite{ahuja2022properties}.

\begin{proposition}
	\label{theorem: additive_intervention1}
	If Assumptions \ref{assum: dgp}, \ref{assm:span}, and \ref{assum: analytic_measure} hold, then the encoder that solves  equation \eqref{eqn:intv_identity} (with $\Delta^{'}$ s.t. $\mathsf{dim}\Big(\mathsf{span}\big(\Delta^{'}\big)\Big)  = d$) identifies true latents up to an invertible affine transform, i.e. $\hat{z} = A z + c$, where $A \in \mathbb{R}^{d\times d}$ is an invertible matrix  and $c \in \mathbb{R}^d$ is an offset.
\end{proposition}

\begin{proof}
	We simplify the identity in equation \eqref{eqn:intv_identity} as follows. 
	\begin{equation}
	\begin{split}
	f(x) + \delta_{i}^{'} = f(\tilde{x}_k) \\ 
    f\circ g(z) + \delta_{i}^{'} = f\circ g (\tilde{z}_k) \\
	a(z) + \delta_{i}^{'} = a(\tilde{z}_k) \\ 
	a(z) + \delta_{i'}^{'} = a(z + \delta_i) \\
	\end{split}
	\label{proof_eqn1}
	\end{equation}
	
	In the above simplification, we use the following observation. Since $x$ and $\tilde{x}_k$ are generated from $g$ and $g$ is injective, we can substitute  $x=g(z)$ and $\tilde{x}_k = g(\tilde{z}_k)$, where $\tilde{z}_k = z + \delta_k$. 
	
	For simplicity denote the last line in above equation \eqref{proof_eqn1} as 
	\begin{equation}
	a(z) + b^{'} = a(z + b).
	\label{proof_eqn2}
	\end{equation}
	
	We take gradient of the function in the LHS and RHS of the above equation \eqref{proof_eqn2} separately w.r.t $z$. Consider the $j^{th}$ component of $a(z+b)$ denoted as $a_{j}(z+b)$. We first take the gradient of $a_{j}(z+b)$ w.r.t $z$ below.
	
	\begin{equation}
	\nabla_z a_j(z+b) = \Big(\frac{dy}{dz}\Big)^{\mathsf{T}}\nabla_y a_j(y),
	\end{equation}
	where $y=z+b$, $\nabla_y a_j(y)$ is the gradient of $a_j$ w.r.t $y$ and  $\frac{dy}{dz}$ denotes the Jacobian of $y$ w.r.t $z$. We simplify the above further to get
	\begin{equation}
	\nabla_z a_j(z+b)  = \nabla_y a_j(y) = \nabla_{y} a_j(z+b). 
	\end{equation}
	
	We can write the above for each component of $a$ as follows.  
	
	\begin{equation}
	\begin{split}
	&  \big[\nabla_z a_1(z+b), \cdots, \nabla_z a_{d}(z+b)\big] = \big[\nabla_{y} a_1(z+b), \cdots,  \nabla_{y} a_d(z+b)\big] \\ 
	&= [\nabla_{y} a_1(z+b), \cdots,  \nabla_{y} a_d(z+b)]  = J^{\mathsf{T}}(z+b),
	\label{eqn: grad_LHS}
	\end{split}
	\end{equation}
	
	where $J(z+b)$ is the Jacobian of $a$ computed at $z+b$. 
	We equate the gradient of LHS and RHS in \eqref{proof_eqn2}  to obtain
	\begin{equation} 
	\begin{split}
	a(z+b) = a(z) +b^{'} \implies  J^{\mathsf{T}}(z+b) - J^{\mathsf{T}}(z) = 0. 
	\end{split}
	\label{theorem3_proof:eqn2_n1}
	\end{equation}
	
	Consider row $j$ of this identity. For each $z \in \mathbb{R}^{d}$
	
	\begin{equation}
	\nabla a_j(z+b) - \nabla a_j(z) =0   \implies \begin{bmatrix}\nabla^{2}_{1}a_{j}(\theta_1) \\ 
	\nabla^{2}_{2}a_{j}(\theta_2)  \\
	\vdots \\ 
	\nabla^{2}_{d}a_{j}(\theta_d) \end{bmatrix}(b) = 0.
	\label{theorem3_proof:eqn2_n3}
	\end{equation}
	where $\nabla^{2} a_j$ is the Hessian of $a_j$ and  $   \nabla^{2}_{k}a_{j}(\theta_k)$ corresponds to the $k^{th}$ row of the Hessian matrix. Note that in the above expansion there is a different $\theta_k$ for each row (mean value theorem applied to each component of $\nabla a_j$ yields a different point $\theta_k$ on the line joinining $\tilde{z}$ and $\tilde{z} + b$). From Assumption \ref{assum: analytic_measure} it follows that $\nabla^{2}_{k}a_{j}(\theta_{k})(b)= 0$ over a set with non-zero measure. Since $a_j$ is analytic $\nabla^{2}_{k}a_{j}(z)(b)$ is also analytic (each component of the vector is a weighted sum of analytic functions). Therefore,   we can conclude that  $\nabla^{2}_{k}a_{j}(z)(b) = 0$ for all $z$ (follows from \cite{mityagin2015zero}). We can make the same argument for each component $k$ and conclude that  $\nabla^{2}a_{j}(z)(b) = 0$. From the identity in equation \eqref{eqn:intv_identity}, it follows that $\nabla^{2}a_j(z)(\delta_{j}) = 0$ for all $j \in \{1, \cdots, d\}$ and since the set $\Delta=\{\delta_{1}, \cdots, \delta_{d}\}$ is linearly independent $\nabla^{2}a_j(z)=0$ for all $z$.  This implies $a(z) = Az+c$. 
	
	We substitute this in equation \eqref{proof_eqn1} to get 
	$A\Delta = \Delta^{'}$, where $\Delta$ is the matrix of true perturbations and $\Delta^{'}$ is the matrix of guessed perturbations (recall we stated above that we use $\Delta,\Delta^{'}$ as sets and matrices interchangeably). We now need to show that $A$ is invertible. Suppose $A$ was not invertible, which implies the rank of $A \leq n-1$. Following Assumption \ref{assm:span}, rank of $\Delta$ is $n$. Note that rank of $\Delta^{'}$ is also $n$. Note that if $E=FG$, where $E$, $F$, $G$ are three matrices, then $\mathsf{rank}(E) \leq \min \{\mathsf{rank}(F), \mathsf{rank}(G)\}$. Following this identity, $\mathsf{rank}(\Delta^{'}) \leq n-1$, which is a contradiction. Therefore, $A$ has to be invertible. This completes the proof. 
	
	
\end{proof}

\begin{theorem}
	\label{theorem: indiv_level_intervention1}
	If Assumptions \ref{assum: dgp}-\ref{assm: sparse1} hold and the number of perturbations per example equals the latent dimension, $m=d$, then the encoder that solves  equation \eqref{eqn:intv_identity} (with $\Delta^{'}$ as one-sparse and $\mathsf{dim}\Big(\mathsf{span}\big(\Delta^{'}\big)\Big)  = d$) identifies true latents up to permutation and scaling, i.e. $\hat{z} = \Pi \Lambda z + c$, where $\Lambda \in \mathbb{R}^{d\times d}$ is an invertible diagonal matrix, $\Pi \in \mathbb{R}^{d\times d}$ is a permutation matrix and $c$ is an offset.
\end{theorem}

\begin{proof}
	Since Assumptions \ref{assum: dgp}, \ref{assm:span}, and \ref{assum: analytic_measure} hold, we can use  Proposition \ref{theorem: additive_intervention} to obtain that any solution to equation \eqref{eqn:intv_identity} achieves affine identification guarantees w.r.t the true latents, i.e. $\hat{z} = Az+c $, where $\hat{z} =f(x)$, $z$ is the inverse image of $x$ ($x=g(z)$), $A\in \mathbb{R}^{d\times d}$ is an inverible matrix and $c \in \mathbb{R}^d$ is the offset vector.  
	
	Define $e_i = [0, \cdots, 1_i, \cdots 0]$ as the vector, which takes a value $1$ at $i^{th}$ component and $0$ everywhere else. Without loss of generality set of true perturbations is $\Delta = \{b_1e_1, \cdots, b_de_d\}$. Note that all $b_i$'s are non-zero as the span of $\Delta$ has a dimension $d$.
	
	Denote the corresponding set of guesses from the agent are $\Delta^{'}= \{c_1e_{\pi(1)}, \cdots, c_d e_{\pi(d)}\}$, where $\pi:\{1, \cdots, d\} \rightarrow \{1, \cdots, d\}$ is a  map used by the agent to guess the coordinate impacted by the perturbation. Note that since $\Delta^{'}$ spans $d$ dimensions $\pi$ has to be a bijection $c_j$'s are non-zero as the span of $\Delta^{'}$. 
	
	Take $b_{j}e_{j} \in \Delta$ and the corresponding guess $c_k e_k$ and substitute it in the relation $\hat{z} = Az+c$ to get 
	\begin{equation}
	\begin{split}
	& \hat{z} = Az+c \\ 
	& \hat{z} + c_ke_k = A(z+b_j e_j)+c \\ 
	& c_ke_k = b_j Ae_j \\ 
	& \frac{c_k}{b_j} e_k =  Ae_j.
	\end{split}
	\end{equation}

	Since $\pi$ is a bijection for every $j$ there is a unique $k$ in the RHS. From the above equation, we gather that the $j^{th}$ column of $A$ is $  \frac{c_k}{b_j} e_k$. We apply this to all the columns and conculde that $\hat{z} = \Pi \Lambda z + c$, where $\Lambda$ is a diagonal matrix and $\Pi$ is a permutation matrix decided based on the bijection $\pi$ ($(\Pi_k = e_{\pi(k)})$, where $\Pi_k$ is the $k^{th}$ colum of the matrix). 
\end{proof}

\begin{theorem}
	\label{theorem: blockwise_intervention1}
	If Assumptions \ref{assum: dgp}-\ref{assum: analytic_measure}, \ref{assum:block_level_intervention}, \ref{assum: span_agent} hold,  then the encoder that solves  equation \eqref{eqn:intv_identity} (where $\Delta^{'}$ is $p$-sparse, $\mathsf{dim}\Big(\mathsf{span}\big(\Delta^{'}\big)\Big)  = d$) identifies true latents up to permutation and block-diagonal transforms, i.e. $f(x)=\hat{z} = \Pi\tilde{\Lambda} z + c$, where $\tilde{\Lambda} \in \mathbb{R}^{d\times d}$ is an invertible block-diagonal matrix with blocks of size $p\times p$, $\Pi \in \mathbb{R}^{d\times d}$ is a permutation matrix and $c \in \mathbb{R}^d$ is an offset.
\end{theorem}

\begin{proof}
	Since Assumptions \ref{assum: dgp}, \ref{assm:span}, and \ref{assum: analytic_measure} hold, we can use  Proposition \ref{theorem: additive_intervention} to obtain that any solution to equation \eqref{eqn:intv_identity} achieves affine identification guarantees w.r.t the true latents, i.e. $\hat{z} = Az+c $, where $\hat{z} =f(x)$, $z$ is the inverse image of $x$ ($x=g(z)$), $A\in \mathbb{R}^{d\times d}$ is an inverible matrix and $c \in \mathbb{R}^d$ is the offset vector.  
	
	We start the proof by assuming that the agent knows the blocks that are impacted under each perturbation, i.e., for each $i \in \mathcal{I}$, the agent knows the block of the latents that are impacted denoted as  $\mathcal{A}_i$. We relax this assumption later. 
	
	Following Assumption \ref{assum:block_level_intervention}, we know that perturbations are $p$-sparse, blockwise and non-overlapping.  Without loss of generality, we can assume that the different groups on which perturbations in $\Delta$ act are given as $\{1, \cdots, p\}$, $\{p+1, \cdots, 2p\}$ and so on. Consider a perturbation $\delta_i$, which belongs to Group $1$ and impacts the latents $\{1, \cdots, p\}$.  For this perturbation, the agent selects $\delta_{i}^{'}$, which shares the same sparsity pattern. Therefore, that the first $p$ elements of $\delta_{i}^{'}$ and $\delta_{i}$ are both non-zero and the rest of the elements are zero. Under these assumptions, we can write the relationship between true and guessed perturbations as follows.
	
	\begin{equation}
	\begin{split}
	&     \hat{z} + \delta_{i}^{'}  = A(z + \delta_{i}) + c \\ 
	&    \delta_{i}^{'} = A \delta_{i}
	\end{split}
	\label{proof_block_eq2}
	\end{equation}
	
	Denote the first $p$ elements of row $k$ of matrix $A$ as $a_k[1:p]$ and the first $p$ elements of the vector $\delta_i$ as $\delta_i[1:p]$. For $k>p$, we use the equation \eqref{proof_block_eq2} to get $a_{k}[1:p]^{\mathsf{T}}\delta_{i}[1:p] = 0$.   
	
	For all perturbations in Group $1$, we can write the same condition, i.e., $a_{k}[1:p]^{\mathsf{T}}\delta_{i}[1:p] = 0$. Since the perturbations in Group 1  span a $p$ dimensional space (following Assumption \ref{assm:span}, \ref{assum:block_level_intervention}), we get that  $a_{k}[1:p]=0$. Therefore, $a_{k}[1:p]=0$ for all $k>p$. 
	
	Let $q$ denote the number of perturbations in Group $1$, where $q\geq p$. For all $k\leq p$ we can solve for the first $p \times p $ block using the perturbations guessed by the agent and the true perturbations in Group $1$. Denote the first $p\times p $ block of $A$ as $A[1:p, 1:p]$ and the first $p$ components of the $q$ perturbations in Group $1$ as $\Delta[1:p, 1:q]$. Similarly, the first $p$ components of the $q$ perturbations guessed by the learner is denoted as $\Delta^{'}[1:p, 1:q]$. We now need to show that the block $A[1:p, 1:p]$ is invertible. From the above equation in \eqref{proof_block_eq2}, we get 
	$$A[1:p, 1:p]\Delta^{'}[1:p, 1:q] = \Delta[1:p, 1:q], $$ 
	
	where $q$ is the number of perturbations in Group $1$. 
	
	Since rank of $\Delta[1:p, 1:q]$ and $\Delta^{'}[1:p, 1:q]$ is $p$, the rank of $A[1:p, 1:p]$ cannot be less than $p$ or else it would lead to a contradiction. This shows that $A[1:p, 1:p]$ is invertible. We derived the properties of the first $p$ columns of matrix $A$. For Group 2, we similarly obtain that $A[p+1:2p, p+1:2p]$ is an invertible matrix and rest of the elements in columns $\{p+1, \cdots, 2p\}$ are zero. Due to symmetry of the setting, we can apply the same argument to all the other blocks as well.    Therefore, we conclude that $A$ is block-diagonal and invertible. This leads to the conclusion that $\hat{z} = \tilde{\Lambda} z + c$, where $\tilde{\Lambda}\in \mathbb{R}^{d\times d}$ and $c\in \mathbb{R}^d$. 
	
	So far we assumed that the agent knows how the interventions in $\{1, \cdots, m \}$  impact the blocks  $\{\mathcal{A}_1, \cdots, \mathcal{A}_m\}$. Under Assumption \ref{assum: span_agent}, the agent knows the groups of the perturbations only. For perturbations $\{\delta_1, \cdots, \delta_p\}$ in Group $1$ that impact $\{1, \cdots, p\}$, the agent guesses $\{\delta_1^{'}, \cdots, \delta_p^{'}\}$. Note that perturbations in $\{\delta_1^{'}, \cdots, \delta_p^{'}\}$ impact the same block of length $p$ with indices given as $\{\alpha_1, \cdots, \alpha_p\}$.   Recall the first $p$ elements of row $k$ of matrix $A$  and vector $\delta_i$ are denoted as $a_k[1:p]$ and $\delta_i[1:p]$ respectively. There exist $d-p$ rows in $A$
	for which we get $a_{k}[1:p]^{\mathsf{T}}\delta_{i}[1:p] = 0$. Thus $a_{k}[1:p]=0 $ for all these rows. The first $p$ elements of remaining $p$ rows form a square matrix denoted as $A[\alpha_1: \alpha_p, 1:p]$, where $\{\alpha_1,\cdots, \alpha_p\}$ are the indices guessed by the agent for the block corresponding to Group 1.   $A[\alpha_1: \alpha_p, 1:p]$ satisfies
	
	$$A[\alpha_1 : \alpha_p, 1:p]\Delta[1:p, 1:q] = \Delta^{'}[\alpha_1:\alpha_p, 1:q],$$ 
	
	where $ \Delta^{'}[\alpha_1:\alpha_p, 1:q]$ is the matrix of non-zero components of the $q$ perturbation vectors that the agent guesses. Using the same argument as above, we can argue that $A[\alpha_1 : \alpha_p, 1:p]$ is invertible. We have derived the properties of first $p$ columns of $A$. We apply the same argument to other groups as well.  Since the agent selects a set of unique $p$ indices for each group, we obtain that the matrix $A$ can be factorized as a permutation matrix times a block diagonal matrix. The first $p$ rows of the permutation matrix with index $\{1, \dots, p\}$ have ones at locations $\{\alpha_1, \cdots, \alpha_p\}$ and so on.  As a result, we get that $\hat{z} = \Pi\tilde{\Lambda}z + c$.

	This completes the proof. 
\end{proof}

\begin{theorem}
	\label{theorem: intersecT1}
	Suppose Assumptions \ref{assum: dgp}, \ref{assum: analytic_measure}, \ref{assum: span_agent} and \ref{assum1:block_level_intervention} hold. Consider the subsets $\mathcal{I}_1$ and $\mathcal{I}_2$ that satisfy Assumption \ref{assum1:block_level_intervention}. For every pair of blocks,  $\mathcal{B}^{1} \in \mathcal{B}_{\mathcal{I}_1}$ and $\mathcal{B}^{2} \in \mathcal{B}_{\mathcal{I}_2}$, the encoder that solves equation \eqref{eqn:intv_identity} (where $\Delta^{'}$ is $p$-sparse, $\mathsf{dim}\Big(\mathsf{span}\big(\Delta^{'}\big)\Big)  = d$) identifies latents in each of the blocks $\mathcal{B}^{1} \cap \mathcal{B}^{2}$, $\mathcal{B}^{1} \setminus \mathcal{B}^{2}$, $\mathcal{B}^{2} \setminus \mathcal{B}^{1}$ up to invertible affine transforms.
\end{theorem}

\begin{proof}
	Following Assumption \ref{assum1:block_level_intervention}, we know that there exists at least two subsets $\mathcal{I}_1$ and $\mathcal{I}_2$ that satisfy blockwise non-overlapping perturbations.  Like in the previous proof, we start this proof also with the case where the agent knows the exact sparsity pattern in the perturbations. We relax this assumption in a bit. Consider a block  $\mathcal{B}^{1} = \{\beta_1, \cdots, \beta_p\}$ impacted by the perturbations in $\mathcal{I}_1$. Since $\mathcal{I}_1$ is blockwise and non-overlapping, we can follow the analysis in the first part of the previous theorem to get $[\hat{z}_{\beta_1}, \cdots, \hat{z}_{\beta_p}]$ is an invertible affine transform of $[z_{\beta_1}, \cdots, z_{\beta_p}]$.  Hence, the latents in each of the blocks $\mathcal{B}^{1} \in \mathcal{G}_{\mathcal{I}_1}$ are identified up to an afffine transform. Similarly, each block $\mathcal{B}^{2} \in \mathcal{G}_{\mathcal{I}_2}$ is identified up to an affine transform.  Consider an element $i \in \mathcal{B}^{1}\cap\mathcal{B}^{2}$. $\hat{z}_i$ can be expressed as an affine transform of two different blocks of latents $z^1$ and $z^2$. $z^{1}$ and $z^{2}$ share some components, we denote them as $z^{12}$. The components exclusive to $z^{1}$ ($z^{2}$) are denoted as $z^{11}$ ($z^{22}$).
	
	We write this condition as follows.
	\begin{equation}
	\begin{split}
	&   \hat{z}_i = a_1^{\mathsf{T}} z^{11} + a_2 z^{12}  + a_3\\ 
	&   \hat{z}_i = b_1^{\mathsf{T}} z^{22} + b_2 z^{12}  + b_3 \\
	& a_{1}^{\mathsf{T}}z^{11} + (a_{2}-b_2)^{\mathsf{T}}z^{12} -b_{1}^{\mathsf{T}}z^{22} = b_3-a_3 
	\end{split}
	\label{proof_thm3_eq1}
	\end{equation}

	If $[a_1,a_2-b_2, b_1]$ is non-zero, i.e., at least one element is non-zero, then the range of LHS is $\mathbb{R}$. But the range of the RHS is a constant. Therefore, for the above to be true $[a_1,a_2-b_2, b_1]=0$ and that implies $a_3=b_3$. As a result, the linear entanglement is now confined to only the intersecting variables $z^{12}$. We can repeat this argument for all elements in $\mathcal{B}^{1}\cap \mathcal{B}^2$.

	In the proof above, we relied on the assumption that the components impacted by each intervention $i \in \mathcal{I}$ are known.  We now relax this assumption and work with assumption that was used in the previous theorem (Assumption \ref{assum: span_agent}), which states that the agent knows the group label of each perturbation. 
	
	Consider the latents in the block $\mathcal{B}^{1}\in \mathcal{G}_{\mathcal{I}_1}$, which we denote as $z^{1}$. We apply Theorem \ref{theorem: blockwise_intervention} to this block. Let the set of estimated latents that affine identify $\mathcal{B}^{1}$ be $\hat{z}^{1} = [\hat{z}_{\alpha_1}, \cdots, \hat{z}_{\alpha_p}]$, where $\{\alpha_1, \cdots, \alpha_p\}$ is the set of indices in $\hat{z}$. We write this as $[\hat{z}_{\alpha_1}, \cdots, \hat{z}_{\alpha_p}] =  A^{1}z^{1} + c^{1}$.  $\tilde{\mathcal{B}}^{1}$ denotes the set of remaining latents not in the block $\mathcal{B}^{1}$. We denote the latents in the block $\tilde{\mathcal{B}}^{1}$ as $z_{c}^{1}$.   Following Theorem \ref{theorem: blockwise_intervention}, we get that the remaining elements of $\hat{z}$ other than $\hat{z}^{1}$, which we denote as $\hat{z}_{c}^{1}$, affine identify the latents $z_{c}^{1}$ in the block  $\tilde{\mathcal{B}}^{1}$.

	Similarly, consider the latents in the group $\mathcal{B}^{2} \in \mathcal{G}_{\mathcal{I}_2}$ denoted as $z^{2}$. $\hat{z}^{2} = [\hat{z}_{\beta_1}, \cdots, \hat{z}_{\beta_p}]$ denotes the latents that affine identify $z^2$. $\tilde{\mathcal{B}}^{2}$ is the set of remaining latents. The remaining elements of $\hat{z}$ other than $\hat{z}^{2}$ are denoted as $\hat{z}_{c}^{2}$. $\hat{z}_{c}^{2}$ affine identifies the latents in the block  $\tilde{\mathcal{B}}^{2}$, which are denoted as $z_{c}^{2}$.

	Recall that the latents $z^{11} \in \mathcal{B}^{1}\setminus \mathcal{B}^{2}$,  $z^{12} \in \mathcal{B}^{1}\cap \mathcal{B}^{2}$, and $z^{22} \in \mathcal{B}^{2}\setminus\mathcal{B}^{1}$.  Consider a latent that is shared between $\hat{z}^{1}$ and $\hat{z}^{2}$. Using the same analysis from equation \eqref{proof_thm3_eq1}, we show that such an element puts a non-zero weight only on $z^{12}$.  Therefore, all the latents shared between  $\hat{z}^{1}$ and $\hat{z}^{2}$ have a non-zero weight on $z^{12}$. 
	Now consider a component of $\hat{z}^{1}$ denoted as $\hat{z}_{\alpha_k}$, which is not present in $\hat{z}^2$. We can write the affine identification condition as 
	
	\begin{equation}
	\hat{z}_{\alpha_k} = c_1^{\mathsf{T}} z^{11} + c_2^{\mathsf{T}} z^{12}  + c_3.
	\label{proof_thm3_eq2}
	\end{equation}
	
	We selected $\hat{z}_{\alpha_k}$, which is not present in $\hat{z}^2$. Since $\hat{z}_{\alpha_k}$ is in $\hat{z}_{c}^2$, we have
	
	\begin{equation}
	\hat{z}_{\alpha_k} = d_1^{\mathsf{T}} z_{c}^{2} + d_3.
	\label{proof_thm3_eq3}
	\end{equation}
	
	If we take a difference of the above two equations \eqref{proof_thm3_eq2} and \eqref{proof_thm3_eq3}, we get that $c_2$ is equal to zero (see the justification below). 
	
	\begin{equation}
	d_1^{\mathsf{T}} z_{c}^{2} + d_3 - c_1^{\mathsf{T}} z^{11} - c_2^{\mathsf{T}} z^{12} -c_3 =0
	\label{proof_thm3_eq4}
	\end{equation}
	
	Note that there is no term associated with $z^{12}$ in equation \eqref{proof_thm3_eq3} as $z_{c}^{2}$ is the set of elements not in $z^2$. Now since the above equation \eqref{proof_thm3_eq3} holds for all $z$, we get $c_2=0$.
	
	From the above analysis we conclude that the latents in $\hat{z}^1$ can be divided into two parts i) the latents that are shared with $\hat{z}^2$; these latents are an affine transform of $z^{12}$, ii) the latents that are not shared with $\hat{z}^2$; these latents are an affine transform of $z^{11}$. We write this condition as 
	\begin{equation} 
	\begin{split}
	\hat{z}^1 = \begin{bmatrix}
	e_1 & 0\\
	0   & e_2
	\end{bmatrix} \begin{bmatrix}
	z^{11} \\
	z^{12}
	\end{bmatrix} + e_3.
	\end{split}
	\label{proof_thm3_eq5}
	\end{equation}
	
	Similarly, we get 
	\begin{equation} 
	\begin{split}
	\hat{z}^2 = \begin{bmatrix}
	f_1 & 0\\
	0   & f_2
	\end{bmatrix} \begin{bmatrix}
	z^{22} \\
	z^{12}
	\end{bmatrix} + f_3.
	\end{split}
	\label{proof_thm3_eq6}
	\end{equation}
	
	We have already discussed above that $f_2=e_2$ and the latter half of $f_3$ corresponding to $z^{12}$ is equals corresponding half of $e_3$.
	
	From the previous theorem, we know that the matrices in the above equations \eqref{proof_thm3_eq5} and \eqref{proof_thm3_eq6} are invertible. Thus if $z^{12}$ has $q$ components, then $e_2$ is $q\times q$ matrix and $e_1$ is $p-q \times p-q$ matrix. This establishes the affine identification of the smaller blocks obtained by intersection of the blocks across two sets of non-overlapping blockwise perturbations. 
	This completes the proof.  
\end{proof}

\begin{theorem}
	\label{cor:block_level_intervention1}
	Suppose Assumptions \ref{assum: dgp}, \ref{assum: analytic_measure}, \ref{assum: span_agent} and \ref{assum:block_level2} hold, then the encoder that solves the identity in equation \eqref{eqn:intv_identity} (where $\Delta^{'}$ is $p$-sparse, $\mathsf{dim}\Big(\mathsf{span}\big(\Delta^{'}\big)\Big)  = d$) identifies true latents up to permutations and scaling, i.e., $\hat{z} = \Pi \Lambda z +c$, where $\Pi \in \mathbb{R}^{d\times d}$ matrix and $\Lambda \in \mathbb{R}^{d\times d}$ is a diagonal matrix.
\end{theorem}

\begin{proof}
	In the above theorem, we use a set of perturbations $\mathcal{I}$ that are $p$-sparse and satisfy the following property. The first $d-(p-1)$ blocks are $\{i, \cdots, i+p-1\}$ from $i=1$ to $i=d-p+1$. The remaining $p-1$ blocks are $\{i, \cdots, (i+p-1) \;\mathsf{mod}\; (d+1)+1 \}$ from $i=d-p+2$ to $d$.  In the $d$ blocks each latent component $i$ is the first element of the block exactly once and also the last component exactly once. 
	
	Construct a partition of perturbations $\mathcal{I}_1$ with continguous blocks $\{k, \cdots, k+p-1\}$ and so on.  Similarly, construct a partition of perturbations $\mathcal{I}_2$ $\{k-(p-1), \cdots, k\}$ and so on. Note that $k$ is the first element of its block in $\mathcal{I}_1$ and it is the last element of its block in $\mathcal{I}_2$. We can apply the Theorem \ref{theorem: intersecT} to conclude that $k^{th}$ component is identified up to scaling and permutation error. We can state the same for all the components. This completes the proof. 
\end{proof}


\subsection{Extensions}
\label{sec: extension}

\subsubsection{Extending Theorem \ref{theorem: indiv_level_intervention}}
In Theorem \ref{theorem: indiv_level_intervention}, we assumed that the number of perturbations $m$ is equal to the number of latent dimensions $d$. Suppose the number of latents is larger than $d$. We sub-sample $d$ distinct perturbation indices from $\{1, \cdots, m\}$. We  solve the identity with the data generated under sub-sampled perturbations in equation \eqref{eqn:intv_identity} with one-sparse guesses. If a solution exists, then we can continue to use the analysis in Theorem \ref{theorem: indiv_level_intervention}.  If a solution does not exist, we sub-sample again and 
solve the identity in equation \eqref{eqn:intv_identity} until we find a solution. 

In Theorem \ref{theorem: indiv_level_intervention}, we assumed that the learner knows that the $\Delta$ in the DGP in equation \eqref{eqn_perturb_obs} is one sparse. Suppose the learner instead guesses that the perturbations are $p$-sparse, where $1<p<d$ and $d\; \mathsf{mod}\; p=0$. In this case, we can use analysis similar to Theorem \ref{theorem: blockwise_intervention} and guarantee blockwise identification, where the blocks are of size $p\times p$.

\subsubsection{Extending Theorem \ref{theorem: blockwise_intervention}}
In this section, we discuss how we can relax Assumption \ref{assum: span_agent}. We first show how to extend Theorem \ref{theorem: blockwise_intervention} to this setting.  In this section, we propose a sparsity test, which would be used to  test if the encoder learned is $p$-sparse or not.  For the rest of the section, we work with blockwise and non-overlapping interventions. We assume that $d \; \mathsf{mod}\; p =0$. Hence, the number of non-overlapping blocks is $r=\frac{d}{p}$.

We take each sample point $(x, \tilde{x}_1, \cdots, \tilde{x}_m)$ and divide it into two parts. We keep the first $d$ perturbations in one set $(\tilde{x}_1, \cdots, \tilde{x}_d)$ to train the encoder and we use the remaining $(\tilde{x}_{d+1}, \cdots, \tilde{x}_m)$ for checking sparsity. We refer to the first $d$ perturbations as training perturbations and the remaining perturbations as validation perturbations.

\begin{assumption}
	\label{assum: delta_ext}
	$\{\delta_i\}_{i=1}^{d}$ is the set of training perturbations, which are $p$-sparse, blockwise and non-overlapping.  $\{\delta_i\}_{i=d+1}^{m}$ is the set of validation perturbations, which are $p$-sparse, blockwise and non-overlapping. The training perturbations span $\mathbb{R}^d$.
\end{assumption}

Consider $d$ perturbations and represent them as follows $\Delta_d$
\begin{equation}
\Delta_d=     \begin{bmatrix}
\Delta_{11} & \Delta_{12} & \cdots, \Delta_{1r}\\ 
\Delta_{12} & \Delta_{22} & \cdots, \Delta_{2r}\\ 
& \;\;\vdots \\ 
\Delta_{r1} & \Delta_{22} & \cdots, \Delta_{rr}\\ 
\end{bmatrix} 
\end{equation}

where $\Delta_{ij}$ is $p\times p$ matrix. Without loss of generality under Assumption \ref{assum: delta_ext}, we can write $\Delta_d$ as a blockdiagonal matrix such that all matrices $\Delta_{ij}=0$ for all $i\not=j$. 

We write the inverse of $\Delta_d$ as

\begin{equation}
\Delta_d^{-1}=     \begin{bmatrix}
\tilde{\Delta}_{11} & \tilde{\Delta}_{12} & \cdots, \tilde{\Delta}_{1r}\\ 
\tilde{\Delta}_{12} & \tilde{\Delta}_{22} & \cdots, \tilde{\Delta}_{2r}\\ 
& \;\;\vdots \\ 
\tilde{\Delta}_{r1} & \tilde{\Delta}_{22} & \cdots, \tilde{\Delta}_{rr}\\ 
\end{bmatrix} 
\end{equation}

\begin{assumption}
	\label{assum: delta_ext1}
	Each element in the matrix along the diagonal of  $\Delta_d^{-1}$ is non-zero, i.e., $\forall k \in \{1, \cdots, r\}, \forall i, j \in \{1, \cdots, p\}, \tilde{\Delta}_{kk}[i,j]\not=0$
\end{assumption}

\begin{assumption}
	\label{assum: deltap_ext}
	The set of interventions guessed by the learner $\Delta^{'}$ contains $d$ perturbations, which are $p$-sparse, blockwise and non-overlapping.
\end{assumption}

We write the $d$ corresponding perturbations that the agent guesses in the form of a matrix as 
\begin{equation}
\Delta_d^{'} =     \begin{bmatrix}
\Delta_{11}^{'} & \Delta_{12}^{'} & \cdots, \Delta_{1r}^{'}\\ 
\Delta_{12}^{'} & \Delta_{22}^{'} & \cdots, \Delta_{2r}^{'}\\ 
& \;\;\vdots \\ 
\Delta_{r1}^{'} & \Delta_{22}^{'} & \cdots, \Delta_{rr}^{'}\\ 
\end{bmatrix} 
\end{equation}

where $\Delta_{ij}^{'}$ is a $p\times p$ matrix.

Define an indicator mask underlying matrix $\Delta^{'}$; it takes a value one wherever there is a non-zero entry and zero otherwise. Define the set of all the masks for $\Delta^{'}$ that satisfy the above assumption (Assumption \ref{assum: deltap_ext}) as $\mathcal{M}=\{1, \cdots, n_{\mathsf{masks}}\}$. Now under the Assumption  \ref{assum: delta_ext}, we get that the validation perturbations are blockwise and non-overlapping as well (though they are not required to span the blocks). We now formalize a simple iterative procedure in which the learner searches over masks that are compliant with the assumption above (Assumption \ref{assum: delta_ext}). The procedure relies on a sparsity test that we describe next.

In the sparsity test, we take a trained encoder and check if for each of the perturbations in the validation set, at most $p$ components change. If for any perturbation more than $p$ estimated components change, then the encoder fails the test. 

\textbf{Joint mask search and encoder learning}
\begin{itemize}
	\item  Select candidate mask $i$ from $\mathcal{M}$. Fill the non-zero entries with random values from some distribution $\mathbb{P}_M$ (we assume that $\mathbb{P}_M$ has a continuous probability density function) to generate a candidate $\Delta^{'}$ 
	\item Solve the identity in equation \eqref{eqn:intv_identity} using samples from the perturbations selected in the step above  $\Delta^{'}$. Check for $p$-sparsity on the set of validation perturbations. If the solution is at most $p$-sparse on all the validation perturbations, then select the encoder. If the solution fails, then $i=i+1$ and go to step one.
\end{itemize}

The mask search procedure described above requires brute force search over many masks. Even though the procedure is computationally intractable, we use it to  demonstrate (see Theorem \ref{theorem: intersecT1} below) that knowledge of sparsity can suffice.

\begin{theorem}
	\label{theorem: intersecT1}
	Suppose Assumptions \ref{assum: dgp}, \ref{assum: analytic_measure}, \ref{assum: delta_ext}, \ref{assum: delta_ext1}, and \ref{assum: deltap_ext} hold, then an encoder that is output learned following the joint mask search and encoder learning procedure above identifies latents up to permutation and block-diagonal transforms with probability one. 
\end{theorem}

\begin{proof} 
	
	We take the encoder $f(x)$ learned from joint mask search and encoder learning procedure described above. Following Assumptions  \ref{assum: dgp}, \ref{assum: analytic_measure}, \ref{assum: delta_ext} and \ref{assum: deltap_ext}, we obtain that $f(x) = \hat{z} = Az +c$, where $x=g(z)$, $A$ is an invertible matrix and $c$ is an offset.  Following the analysis in Proposition \ref{theorem: additive_intervention}, we obtain $A$ matrix is given as $A = \Delta_d^{'}\Delta_d^{-1}$ (substitue $\hat{z} = Az+c$, $\hat{z}+\Delta_d = Az +c + A\Delta_d^{'} $). We index the matrix in terms of the blocks. 
	
	The matrix at location $(i,j)$ is $A_{ij} = \Delta_{ij}^{'}\tilde{\Delta}_{jj}$ (since $\Delta$ is a blockdiagonal matrix, i.e., $\Delta_{ij}=0$ for $i\not=j$ but $\Delta_{ii}\not=0$).  Each column of $\Delta^{'}$ consists of $p$ non-zero entries. Using this and $A_{ij} = \Delta_{ij}^{'}\tilde{\Delta}_{jj}$ we argue next that the number of non-zero entries in each column of $A$ are at least $p$.  We write $A_{ij}[k,q] = \sum_{l}\Delta_{ij}^{'}[k,l]\tilde{\Delta}_{jj}[l,q]$. Since $\Delta_{ij}^{'}[k,l]$ and $\tilde{\Delta}_{jj}[l,q]$ both take non-zero value, the first term in the above summation is non-zero. Since the other terms depend on random variables drawn independently, the probability that the sum equals zero is zero. Therefore, for each of the $p$ indices $k$  where the mask is non-zero, the   $A_{ij}[k,q]$ is non-zero.

	Suppose at least one column block  of $A$, say $jp+1:(j+1)p$,  contains two columns which exhibit a different sparsity pattern.   Since there are at least two columns which share a different sparsity pattern, there is at least one row where only one of them is zero and other is non-zero. 
	Therefore, in this column block we have at least $p+1$ rows which have at least one non-zero element. The encoder passed the sparsity test, i.e., for all the perturbations on blocks of the form  $jp+1:(j+1)p$ we have at most $p$-sparse output. Therefore, at least one of the $p+1$ rows has to multiply with the block and output a zero, which is a zero probability event (since the non-zero elements of $A$ matrix are all continuous random variables).  Thus if any contiguous block has different sparsity pattern across columns, then the encoder is selected with probablity zero. Thus from this we can conclude that for a selected encoder, each column block exhibits a sparsity pattern that is same across all the columns in the block. To ensure that $A$ is an invertible, all blocks exhibit a non-overlapping sparsity pattern.  
	Therefore, $A$ is permutation times a diagonal matrix. We now illustrate what choices of $\Delta^{'}$ lead to an $A$ that passes the sparsity test.   If for every $i$ there exists a unique $j$ for which $\Delta_{ij}^{'}$ is invertible and every other value of $j$, $\Delta_{ij}^{'}=0$, then $A$ is permutation times a diagonal matrix. This completes the proof.

\end{proof}


In this section, we showed that we we do not need to make Assumption \ref{assum: span_agent} and the knowledge of sparsity suffices to do blockwise identification. Following similar analysis as above, we can extend Theorem \ref{cor:block_level_intervention} as well. 

\subsubsection{Connection with causal interventions} In the DGP in equation \eqref{eqn_perturb_obs}, we assumed that $Z$ is sampled from any distribution $\mathbb{P}_Z$. We now consider a special case, where $Z = [Z_1, \cdots, Z_d]$ follows a certain structural causal model $\mathcal{S}$ given as 

\begin{equation}
Z_{i} \leftarrow f_{i}(\mathsf{Pa}(Z_i), U_i), \forall i \in \{1,\cdots, d\}  
\label{eqn_scm1}
\end{equation}

where $Z_i$ is generated from its parent variables denoted by $\mathsf{Pa}(Z_i)$ using the mechanism $f_i: \Pi_{\mathsf{Pa}(Z_i)}\mathcal{Z}_i \times \mathcal{U}_i \rightarrow \mathbb{R}$, which also takes the noise variable $U_i$ as input. The support of $Z_i$ is denoted by $\mathcal{Z}_i$  and that of $U_i$ is denoted by $\mathcal{U}_i$.   Suppose we perturb $Z_k$.  Under this perturbation all the latent variables for which $Z_k$ is an ancestor are going to be also affected, while keeping the rest of the variables unchanged.

Post the perturbation, the immediate children of $Z_k$ are affected and then their children and so on.  Therefore, it is reasonable to assume that we first observe the impact of perturbation on $Z_k$ itself and eventually observe the impact on child variables.   Consider a sample point $[(z_1, \cdots z_d ), (x_1, \cdots, x_n)]$ generated by equation \eqref{eqn1}.  The different observations under perturbations are  

\begin{itemize}
	\item \textbf{Pre perturbation:} $[(z_1, \cdots z_k, \cdots, z_d ), (x_1, \cdots, x_n)]$
	\item \textbf{At the time of perturbation:}  $[(z_1, \cdots z_k+\delta, \cdots, z_d ), (x_1^{'}, \cdots, x_n^{'})]$
	\item \textbf{Sufficiently long after the perturbation:}  $[(z_1, \cdots z_k+\delta, \cdots, z_d^{''} ), (x_1^{''}, \cdots, x_n^{''})]$
\end{itemize}

In the above, the latent of the sample pre perturbation and at the time of perturbation only differ in the perturbed components. However, when sufficient period has passed, other latent variables that are on the downstream path from $Z_k$ also change.   In this work, we only deal with original samples and the samples at the time of perturbation. In causal interventions, we assume access to samples before perturbation and those generated sufficiently long after the perturbation.

\newpage 
\subsection{Supplementary materials for experiments}
\label{sec : exp_supp}
\paragraph{Loss function, architecture, and other hyperparameters}
In all the experiments we optimized equation \eqref{eqn:mse} with square error loss. The encoder $f$ was an MLP with two hidden layers of size $100$ for the low-dimensional synthetic experiments and a ResNet-18 \citep{He2015} for the image-based experiments. We used the Adam optimizer \citep{adam} with a learning rate of $0.005$ with batches of $10000$ examples for $2000$ epochs for the low-dimensional synthetic experiments; the image-based models were trained online with a learning rate of $1e-4$ and a batch size of 100. 

\paragraph{Evaluation metrics} Blockwise MCC (BMCC) is a natural extension of MCC; instead of measuring correlations between true and estimated latents.
We compute the $R^2$ score between every pair of blocks impacted under true perturbation and the guessed perturbation. We find the optimal matching between pairs of blocks to maximize the average $R^2$ score between the matched blocks. We report the $R^2$ score under the optimal matching in Table \ref{table_synth}. 

\paragraph{Supplementary figures} In Figure \ref{fig:supp_1}, we plot the predicted latents against the true latent value for two of the ten latent dimensions (the two dimensions that we plot are randomly selected) when we perturb one component at a time (setting corresponds to the paragraph on non-overlapping perturbations in Section \ref{sec: expmts}). The plot shows a linear relationship between the true and the predicted latent; note that there are different slope and intercept for the different latents. The slope depends on the ratio between the change in the true latents and the predicted latent. In Figure \ref{fig:supp_2}, we plot the predicted latents against the true latent value for two of the ten latent dimensions (the two dimensions that we plot are randomly selected) when we perturb a block of two components at a time and the blocks overlap (setting corresponds to the paragraph on overlapping perturbations in Section \ref{sec: expmts}). In Figure \ref{fig:supp_3}, we show a full set of images for the experiment shown in Figure \ref{fig: ball_env}. 
\begin{figure}
	\centering
	\includegraphics{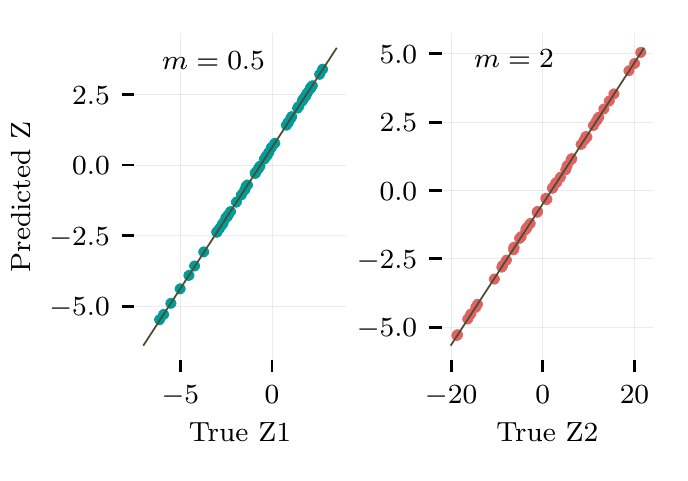}
	\caption{Regression of predicted latent values against true latent values for componentwise perturbations ($d=10$).}
	\label{fig:supp_1}
\end{figure}

\begin{figure}
	\centering
	\includegraphics{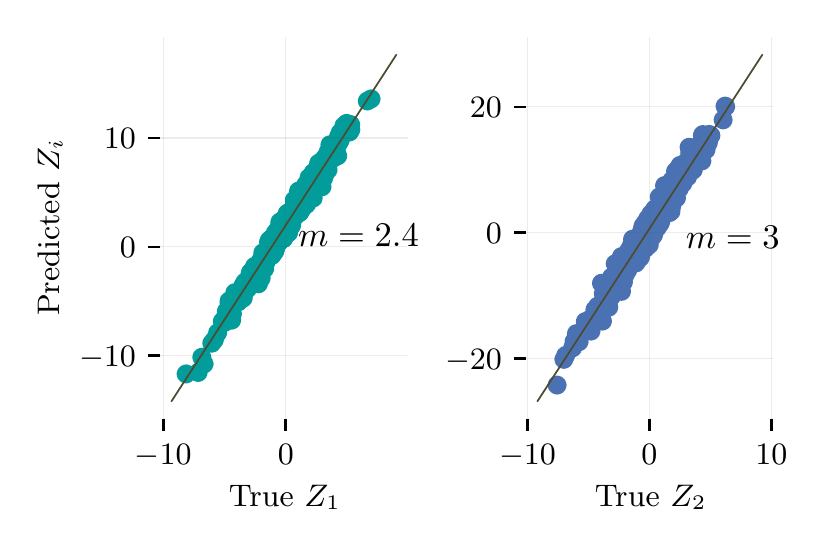}
	\caption{Regression of predicted latent values against true latent values for overlapping perturbations ($d=10$).}
	\label{fig:supp_2}
\end{figure}

\begin{figure}
	\centering
	\includegraphics[width=\textwidth]{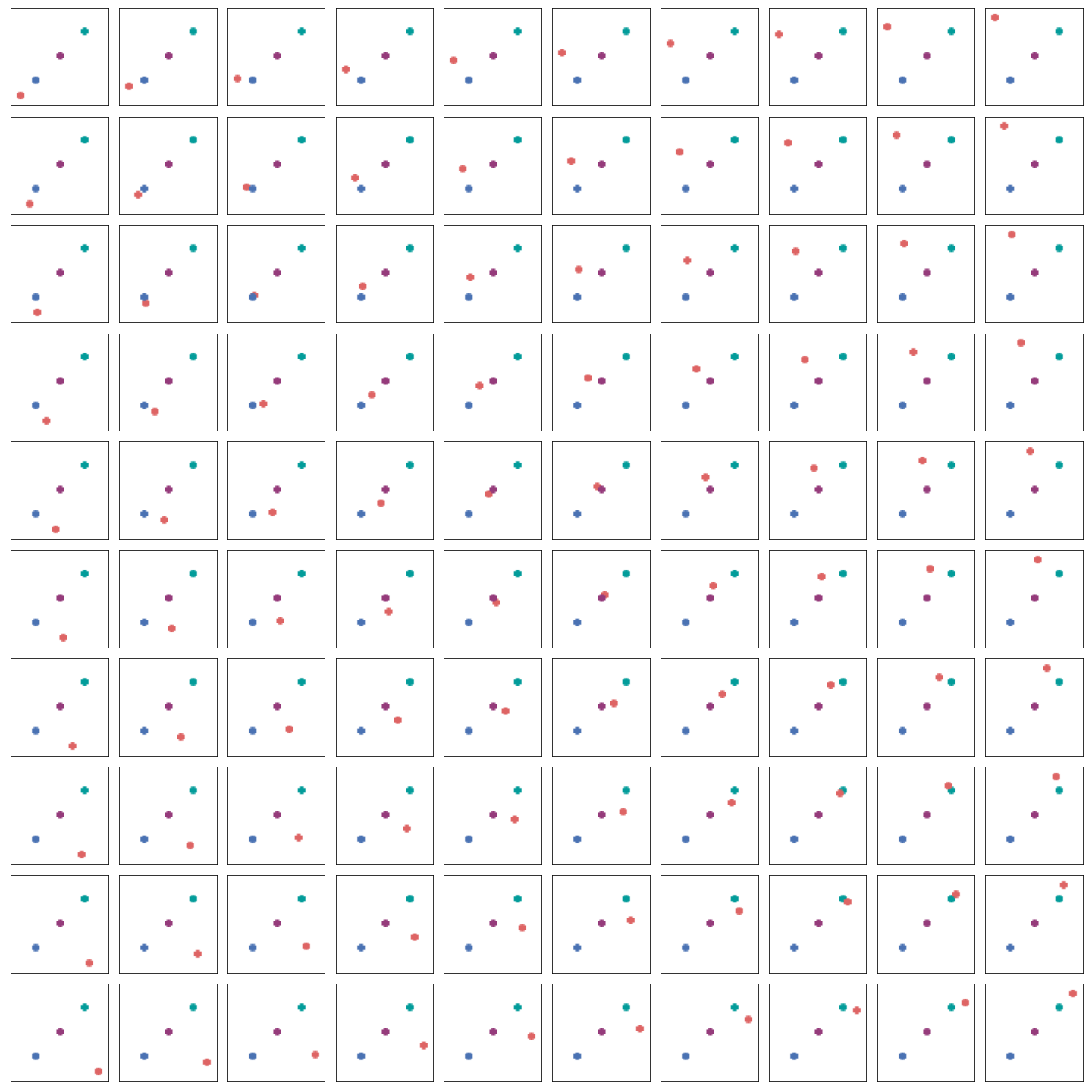}
	\caption{Full set of images for the experiment shown in Figure \ref{fig: ball_env} used to render the supplementary animation. The three balls on the diagonal are stationary throughout and the fourth ball is moved across a $10 \times 10$ grid; we get the associated network predictions and animate them to show the predicted movement of the stationary balls in the attached animation.}
	\label{fig:supp_3}
\end{figure}

\newpage

\subsubsection{Stationary point}

Recall our learning objective is to minimize the objective given in Equation \ref{eqn:mse}. We use a deep network, $\tilde{f}(\cdot;\theta)$ parameterized by
$\theta$ as our encoder and we can rewrite Equation \ref{eqn:mse} as a loss function that depends on our choice of $\theta$ and $\Delta'$ (the learner's guess for the offsets),

\begin{equation}
\mathcal{L}(\theta, \Delta^{'}) = \mathbb{E} \Big[ \Big\| f(\tilde{x}_k; \theta) - f(x; \theta)  - \delta_k^{'}\Big\|^2 \Big] = \mathbb{E} \Big[ \sum_{j}\Big( f_j(\tilde{x}_k; \theta) - f(x; \theta)  - \delta_k^{'}\Big)^2 \Big]
\end{equation}

We take the partial derivative of the loss with respect to one of the parameters $\theta_i$ and obtain 

\[
\frac{\partial \mathcal{L}(\theta)}{\partial \theta_{i}}  = \mathbb{E}_{x, \tilde{x}}\left[\sum_{j}\underbrace{(f_j(\tilde{x}; \theta) - f_j(x; \theta)  - \delta_j^{'})}_{=:e_j(x, \tilde{x}_k, \theta)}\underbrace{(\frac{\partial f_j(\tilde{x})}{\partial \theta_{i}}- \frac{\partial f_j(x)}{\partial \theta_{i}})}_{=:\phi_j(x,\tilde{x}_k, \theta_i)}\right]
\]

Suppose we learn a function $\tilde{f}$ for which $e_j(x, \tilde{x}_k, \theta)$ is independent of $x$ and $\tilde{x}$ and we denote it as $e_j(\theta)$ for all $j \in \{1, \cdots, d\}$. Under this assumption, we simplify the above expression as follows. 

\[
\frac{\partial \mathcal{L}(\theta)}{\partial \theta_{i}}  = \sum_{j}e_{j}(\theta)\mathbb{E}_{x, \tilde{x}}\Big[\phi_j(x,\tilde{x}_k, \theta_i)\Big] = \sum_{j} e_{j}(\theta) \mu_{j}(\theta)
\]
where $\mu_{j}(\theta) = \mathbb{E}_{x, \tilde{x}}\Big[\phi_j(x,\tilde{x}_k, \theta_i)\Big]$. $\mu_{j}(\theta)$ measures the expected difference in the guessed perturbation for the component $j$ when parameter $\theta_i$ of the neural network is changed. If the impact of change in the parameter is similar on average across all the components, i.e.,  $\mu_j(\theta) = \mu_{k}(\theta) = \mu(\theta)$ for all $j \not=k$, then 

\[
\frac{\partial \mathcal{L}(\theta_i)}{\partial \theta_{i}}  = \sum_{j}e_{j}(\theta)\mathbb{E}_{x, \tilde{x}}\Big[\phi_j(x,\tilde{x}_k, \theta_i)\Big] = \mu(\theta) \sum_{j} e_{j}(\theta_i) 
\]

Under these conditions, this is a stationary point if $\sum_{j} e_{j}(\theta_i)=0$ for all $\theta_i$. Empirically we observe that if $j$ is perturbed by $c$, then $e_{j}(\theta)= \frac{c}{2} $ and other components $k \not =j$, $e_{j}(\theta)= \frac{-c}{2(n_{\mathsf{balls}}-1)}$. If we substitute this in the equation above, we find that the partial derivative is zero. Since this holds for all the components $\theta_i$, we can conclude that the point observed empirically is a stationary point. Under the assumption that $e_j(x, \tilde{x}_k, \theta)$ is independent of $e_j(x, \tilde{x}_k, \theta)$, we can follow the analysis presented in proof of Theorem \ref{theorem: indiv_level_intervention}, we get $\hat{z} = Az +c$. If $z$ changes by $[c, 0, \cdots, 0]$, then $\hat{z} = [\frac{c}{2}, -\frac{c}{2(n_{\mathsf{balls}}-1)}, \cdots, -\frac{c}{2(n_{\mathsf{balls}}-1)}]$. We use this to obtain $A[i,j] = \frac{-1}{2(n_{\mathsf{balls}}-1)}$, where $i\not=j$ and $A[i,i] = \frac{1}{2}$. If $n_{\mathsf{balls}}=\infty$, then $A$ is a diagonal matrix, which implies that the MCC is one. In the discussion above, we assumed that the learner knows the component that changes. If the learner does not know the component that changes, then that introduces permutation errors as well.

\end{document}